\documentclass{article}

\usepackage{xspace,amsmath,amssymb,amsthm,amsfonts,epsfig,syntonly,dsfont,pifont,natbib}
\usepackage{cite,bm,color,url,textcomp}
\usepackage{algorithmicx,algorithm,algpseudocode}
\usepackage{epstopdf}
\usepackage{empheq,float}
\usepackage{graphicx,subfigure,balance}
\usepackage{multirow}
\usepackage{float}

\newtheorem{assumption}{Assumption}
\newtheorem{proposition}{Proposition}

\newtheorem{lemma}{Lemma}
\newtheorem{corollary}{Corollary}
\newtheorem{theorem}{Theorem}
\newtheorem{definition}{Definition}
\theoremstyle{definition}

\usepackage{wrapfig}

\usepackage{times}
\usepackage[margin=1.05in]{geometry}

\usepackage[utf8]{inputenc} 
\usepackage[T1]{fontenc}    
\usepackage{hyperref}       
\usepackage{url}            
\usepackage{booktabs}       
\usepackage{nicefrac}       
\usepackage{microtype}      

\usepackage{scrextend}

\graphicspath{{../figs/}}

\title{\huge{Almost Tune-Free Variance Reduction}}
\author{ 
	Bingcong Li*  ~~ Lingda Wang$^\dagger$ ~~ Georgios B. Giannakis* 	\vspace{0.1cm} \\	 
	 * \textit{University of Minnesota - Twin Cities, Minneapolis, MN 55455, USA} \\
	 \texttt{\{lixx5599, georgios\}@umn.edu} \\
	 $\dagger$ \textit{University of Illinois at Urbana-Champaign, Urbana, IL 61801, USA} \\
	 \texttt{lingdaw2@illinois.edu} 
	 }
\vspace{0.3cm} 

\begin{document}

\maketitle
\begin{abstract}
The variance reduction class of algorithms including the representative ones, SVRG and SARAH, have well documented merits for empirical risk minimization problems. However, they require grid search to tune parameters (step size and the number of iterations per inner loop) for optimal performance. This work introduces `almost tune-free' SVRG and SARAH schemes equipped with i) Barzilai-Borwein (BB) step sizes; ii) averaging; and, iii) the inner loop length adjusted to the BB step sizes. In particular, SVRG, SARAH, and their BB variants are first reexamined through an `estimate sequence' lens to enable new averaging methods that tighten their convergence rates theoretically, and improve their performance empirically when the step size or the inner loop length is chosen large. Then a simple yet effective means to adjust the number of iterations per inner loop is developed to enhance the merits of the proposed averaging schemes and BB step sizes. Numerical tests corroborate the proposed methods.
\end{abstract}

\section{Introduction}
In this work, we deal with the frequently encountered empirical risk minimization (ERM) problem expressed as 
\begin{align}\label{eq.prob}
	\min_{\mathbf{x} \in \mathbb{R}^d} f(\mathbf{x}) := \frac{1}{n} \sum_{i \in [n]} f_i(\mathbf{x})	
\end{align}
where $\mathbf{x} \in \mathbb{R}^d$ is the parameter vector to be learned from data; the set $[n]:= \{1,2,\ldots, n \}$ collects data indices; and, $f_i$ is the loss function associated with datum $i$. Suppose that $f$ is $\mu$-strongly convex and has $L$-Lipchitz continuous gradient. The condition number of $f$ is denoted by  $\kappa := L/\mu$. Throughout, $\mathbf{x}^*$ denotes the optimal solution of \eqref{eq.prob}. The standard approach to solve \eqref{eq.prob} is \emph{gradient descent} (GD) \citep{nesterov2004}, which updates the decision variable via 
\begin{align*}
	\mathbf{x}_{k+1} = \mathbf{x}_k - \eta\nabla f(\mathbf{x}_k)
\end{align*}
where $k$ is the iteration index and $\eta$ the step size (or learning rate). For a strongly convex $f$, GD convergences linearly to $\mathbf{x}^*$, that is, $\|\mathbf{x}_k - \mathbf{x}^* \|^2 \leq (c_{\kappa})^k \|\mathbf{x}_0 - \mathbf{x}^* \|^2$ for some $\kappa$-dependent constant $c_\kappa\in(0,1)$ \citep{nesterov2004}. 

In the big data regime, however, where $n$ is large, obtaining the gradient per iteration can be computationally prohibitive. To cope with this, the \emph{stochastic gradient descent} (SGD) reduces the computational burden by drawing uniformly at random an index $i_k \in [n]$ per iteration $k$, and adopting $\nabla f_{i_k}(\mathbf{x}_k)$ as an estimate of $\nabla f(\mathbf{x}_k)$. Albeit computationally lightweight with the simple update 
\begin{align*}
	\mathbf{x}_{k+1} = \mathbf{x}_k - \eta_k \nabla f_{i_k}(\mathbf{x}_k)
\end{align*}
the price paid is that SGD comes with sublinear convergence, hence slower than GD \citep{robbins1951,bottou2018}. It has been long recognized that the variance
$\mathbb{E}[\|  \nabla f_{i_t}(\mathbf{x}_t) - \nabla f(\mathbf{x}_t) \|^2]$ of the gradient estimate affects critically SGD's convergence slowdown.

This naturally motivated gradient estimates with \emph{reduced variance} compared with SGD's simple $\nabla f_{i_k}(\mathbf{x}_k)$. A gradient estimate with reduced variance can be obtained by capitalizing on the finite sum structure of \eqref{eq.prob}. One idea is to judiciously evaluate a so-termed \textit{snapshot gradient} $\nabla f(\mathbf{x}_s)$, and use it as an anchor of the stochastic draws in subsequent iterations. Members of the variance reduction family include SVRG \citep{johnson2013}, SAG \citep{roux2012}, SAGA \citep{defazio2014}, MISO \citep{mairal2013}, SARAH \citep{nguyen2017}, and their variants \citep{konecny2013,lei2017,kovalev2019,li2019l2s}. Most of these algorithms rely on the update $\mathbf{x}_{k+1} = \mathbf{x}_k - \eta \mathbf{v}_k$, where $\eta$ is a constant step size and $\mathbf{v}_k$ is an algorithm-specific gradient estimate that takes advantage of the snapshot gradient. In this work, SVRG and SARAH are of central interest because they are memory efficient compared with SAGA, and have no requirement for the duality arguments that SDCA \citep{shalev2013} entails. Variance reduction methods converge linearly when $f$ is strongly convex. To fairly compare the complexity of (S)GD with that of variance reduction algorithms which combine snapshot gradients with the stochastic ones, we will rely on the incremental first-order oracle (IFO) \citep{agarwal2014}.
\begin{definition}
	An IFO takes $f_i$ and $\mathbf{x} \in \mathbb{R}^d$ as input, and returns the (incremental) gradient $\nabla f_i(\mathbf{x})$.
\end{definition}
For convenience, IFO complexity is abbreviated as complexity in this work. A desirable algorithm obtains an $\epsilon$-accurate solution satisfying $\mathbb{E}[\| \nabla f(\mathbf{x}) \|^2] \leq \epsilon$ or $\mathbb{E}[ f(\mathbf{x}) - f(\mathbf{x}^*) ] \leq \epsilon$ with minimal complexity for a prescribed $\epsilon$. Complexity for variance reduction alternatives such as SVRG and SARAH is ${\cal O}\big((n+\kappa) \ln \frac{1}{\epsilon} \big)$, a clear improvement over GD's complexity ${\cal O}\big(n\kappa \ln \frac{1}{\epsilon} \big)$. And when high accuracy (small $\epsilon$) is desired, the complexity of variance reduction algorithms is also lower than SGD's complexity of ${\cal O}\big(\frac{1}{\epsilon}\big)$. The merits of gradient estimates with reduced variance go beyond convexity; see e.g., \citep{reddi2016,fang2018,cutkosky2019}, but nonconvex ERM are out of the present work's scope.

Though theoretically appealing, SVRG and SARAH entail grid search to tune the step size, which is often painstakingly hard and time consuming. An automatically tuned step size for SVRG was introduced by \citep{barzilai1988} (BB) and \citep{tan2016}. However, since both SVRG and SARAH have a double-loop structure, the inner loop length also requires tuning in addition to the step size. Other works relying on BB step sizes introduce additional tunable parameters on top of the inner loop length \citep{liuclass,yang2019}. In a nutshell, `tune-free' variance reduction algorithms still have desired aspects to investigate and fulfill.  

Along with the BB step sizes, this paper establishes that in order to obtain `tune-free' SVRG and SARAH schemes, one must: i) develop novel types of averaging; and, ii) adjust the inner loop length along with step size as well. Averaging in double-loop algorithms reflects the means of choosing the starting point of the next outer 
loop \citep{johnson2013,tan2016,nguyen2017}. The types of averaging considered so far have been employed as tricks to simplify proofs, while in the algorithm itself the last iteration is the most prevalent choice for the starting point of the ensuing outer loop. However, we contend that different averaging methods result in different performance. And the best averaging depends on the choice of other parameters. In addition to averaging, we argue that the choice of the inner loop length for BB-SVRG in \citep{tan2016} is too pessimistic. Addressing this with a simple modification leads to the desired `almost tune-free' SVRG and SARAH. 

Our detailed contributions can be summarized as follows. 
\begin{enumerate}
	\item[\textbullet] We empirically argue that averaging is not merely a proof trick. It is prudent to adjust averaging in accordance with the step size and the inner loop length.
	\item[\textbullet] SVRG and SARAH are analyzed using the notion of estimate sequence (ES). This prompts a novel averaging that tightens up convergence rate for SVRG, and further improves SARAH's convergence over existing works under certain conditions. Besides tighter rates, our analysis broadens the analytical tool, ES, by endowing it with the ability to deal with SARAH's biased gradient estimate. 
	\item[\textbullet] The theoretical guarantees for BB-SVRG and BB-SARAH with different types of averaging are established and leveraged for performance improvement.
	\item[\textbullet] Finally, we offer a principled design of the inner loop length to obtain almost tune-free BB-SVRG and BB-SARAH. The choice for the inner loop length is guided by the regime that the proposed averaging schemes favor. Numerical tests further corroborate the efficiency of the proposed algorithms.
\end{enumerate}

\textbf{Notation}. Bold lowercase letters denote column vectors; $\mathbb{E}$ represents expectation; $\| \mathbf{x}\|$ stands for the $\ell_2$-norm of $\mathbf{x}$; and $\langle \mathbf{x}, \mathbf{y} \rangle$ denotes the inner product of vectors $\mathbf{x}$ and $\mathbf{y}$.

\section{Preliminaries}\label{sec.intro}

We will first focus on the averaging techniques, whose generality goes beyond BB step sizes. To start with, this section briefly reviews the vanilla SVRG and SARAH, while their BB variants are postponed slightly.

\subsection{Basic Assumptions}
\begin{assumption}\label{as.1}
Each $f_i: \mathbb{R}^d \rightarrow \mathbb{R}$ has $L$-Lipchitz gradient, that is, $\|\nabla f_i(\mathbf{x}) - \nabla f_i(\mathbf{y}) \| \leq L \| \mathbf{x}-\mathbf{y} \|, \forall \mathbf{x}, \mathbf{y} \in \mathbb{R}^d$.
\end{assumption} 
\begin{assumption}\label{as.2}
	Each $f_i: \mathbb{R}^d \rightarrow \mathbb{R}$ is convex.
\end{assumption}
\begin{assumption}\label{as.3}
	Function $f: \mathbb{R}^d \rightarrow \mathbb{R}$ is $\mu$-strongly convex, that is, there exists $\mu > 0$, such that $f(\mathbf{x}) - f(\mathbf{y}) \geq \langle \nabla f(\mathbf{y}), \mathbf{x}-\mathbf{y}\rangle + \frac{\mu}{2} \| \mathbf{x}-\mathbf{y}\|^2,$ $\forall \mathbf{x}, \mathbf{y} \in \mathbb{R}^d$.
\end{assumption} 
\begin{assumption}\label{as.4}
	Each $f_i: \mathbb{R}^d \rightarrow \mathbb{R}$ is $\mu$-strongly convex, meaning there exists $\mu > 0$, so that $f_i(\mathbf{x}) - f_i(\mathbf{y}) \geq \langle \nabla f_i(\mathbf{y}), \mathbf{x}-\mathbf{y}\rangle + \frac{\mu}{2} \| \mathbf{x}-\mathbf{y}\|^2,$ $\forall \mathbf{x}, \mathbf{y} \in \mathbb{R}^d$.
\end{assumption} 
Assumption \ref{as.1} requires each loss function to be sufficiently smooth. One can certainly require smoothness of each individual loss function and refine Assumption \ref{as.1} as $f_i$ has $L_i$-Lipchitz gradient. Clearly $L =\max_i L_i$. By combining with importance sampling \citep{xiao2014,kulunchakov2019}, such a refined assumption can slightly tighten the $\kappa$ dependence in the complexity bound. However, since the extension is straightforward, we will keep using the simpler Assumption \ref{as.1} for clarity. Assumption \ref{as.3} only requires $f$ to be strongly convex, which is weaker than Assumption \ref{as.4}. Assumptions \ref{as.1} -- \ref{as.4} are all standard in variance reduction algorithms.

\subsection{Recap of SVRG and SARAH}\label{sec.recap}

\vspace{-0.6cm}
\begin{minipage}[t]{0.47\textwidth}
\null
\begin{algorithm}[H]
    \caption{SVRG}\label{alg.svrg}
    \begin{algorithmic}[1]
    	\State \textbf{Initialize:} $\tilde{\mathbf{x}}^0 $, $\eta$, $m$, $S$
    	\For {$s=1,2,\dots,S$}
			\State $\mathbf{x}_0^s = \tilde{\mathbf{x}}^{s-1}$
			\State $\mathbf{g}^s =  \nabla f (\mathbf{x}_0^s )$
			\For {$k=0,1,\dots,m-1$}
				\State uniformly draw $i_k \in [n]$ 
				\State $\mathbf{v}_k^s = \nabla f_{i_k} (\mathbf{x}_k^s ) -\nabla f_{i_k} (\mathbf{x}_0^s ) + \mathbf{g}^s $
				\State $\mathbf{x}_{k+1}^s = \mathbf{x}_k^s - \eta \mathbf{v}_k^s$
			\EndFor
			\State select $\tilde{\mathbf{x}}^{s}$ randomly from $\{\mathbf{x}_k^s \}_{k=0}^m$ following $\mathbf{p}^s$ 
		\EndFor
		\State \textbf{Output:} $\tilde{\mathbf{x}}^S$
	\end{algorithmic}
\end{algorithm}
\end{minipage}
\hspace{0.7cm}
\begin{minipage}[t]{0.47\textwidth}
\null
\begin{algorithm}[H]
    \caption{SARAH}\label{alg.sarah}
    \begin{algorithmic}[1]
    	\State \textbf{Initialize:} $\tilde{\mathbf{x}}^0 $, $\eta$, $m$, $S$
    	\For {$s=1,2,\dots,S$}
			\State $\mathbf{x}_0^s = \tilde{\mathbf{x}}^{s-1}$, and $\mathbf{v}_0^s =  \nabla f (\mathbf{x}_0^s )$
			\State $\mathbf{x}_1^s = \mathbf{x}_0^s - \eta \mathbf{v}_0^s $
			\For {$k=1,2,\dots,m-1$}
				\State uniformly draw $i_k \in [n]$ 
				\State $\mathbf{v}_k^s = \nabla f_{i_k} (\mathbf{x}_k^s ) - \nabla f_{i_k} (\mathbf{x}_{k-1}^s )  + \mathbf{v}_{k-1}^s $
				\State $\mathbf{x}_{k+1}^s = \mathbf{x}_k^s - \eta \mathbf{v}_k^s$
			\EndFor
			\State select $\tilde{\mathbf{x}}^{s}$ randomly from $\{\mathbf{x}_k^s \}_{k=0}^m$ following $\mathbf{p}^s$ 
		\EndFor
		\State \textbf{Output:} $\tilde{\mathbf{x}}^S$
	\end{algorithmic}
\end{algorithm}
\end{minipage}
\vspace{0.2cm}

The steps of SVRG and SARAH are listed in Algs. \ref{alg.svrg} and \ref{alg.sarah}, respectively. Each employs a fine-grained reduced-variance gradient estimate per iteration. For SVRG, $\mathbf{v}_k^s$ is an unbiased estimate since $\mathbb{E}[\mathbf{v}_k^s|{\cal F}_{k-1}^s] = \nabla f(\mathbf{x}_k^s)$, where ${\cal F}_{k-1}^s:= \sigma(\tilde{\mathbf{x}}^{s-1}, i_0, i_1, \ldots, i_{k-1})$ is the $\sigma$-algebra generated by $\tilde{\mathbf{x}}^{s-1}, i_1, i_2, \ldots,i_{k-1}$; while SARAH adopts a biased $\mathbf{v}_k^s$, that is, $\mathbb{E}[\mathbf{v}_k^s|{\cal F}_{k-1}^s]= \nabla f(\mathbf{x}_k^s) - \nabla f(\mathbf{x}_{k-1}^s) + \mathbf{v}_{k-1}^s \neq \nabla f(\mathbf{x}_k^s)$. The variance (mean-square error (MSE)) of $\mathbf{v}_k^s$ in SVRG (SARAH) can be upper bounded by quantities that dictate the optimality gap (gradient norm square). 

\begin{lemma}\label{lemma.est_err}
	\citep{johnson2013,nguyen2017} The MSE of $\mathbf{v}_k^s$ in SVRG is bounded as follows
	\begin{subequations}
		\begin{align}
			{\rm SVRG:}&~ \mathbb{E} \big[ \|  \nabla f(\mathbf{x}_k^s) - \mathbf{v}_k^s  \|^2  \big] \leq \mathbb{E} \big[ \| \mathbf{v}_k^s \|^2  \big] \leq 4L \mathbb{E}\big[ f(\mathbf{x}_k^s ) - f(\mathbf{x}^*)\big] + 4L \mathbb{E}\big[ f(\mathbf{x}_0^s) - f(\mathbf{x}^*)\big].  \nonumber 
		\end{align}
		The MSE of $\mathbf{v}_k^s$ in SARAH is also bounded as
		\begin{align}
			{\rm SARAH:~~} & \mathbb{E}\big[ \|  \nabla f (\mathbf{x}_k^s) - \mathbf{v}_k^s \|^2 \big]  \leq \frac{\eta L}{2 - \eta L} \bigg( \mathbb{E}\big[ \| \nabla f(\mathbf{x}_0^s) \|^2 \big] - \mathbb{E}\big[ \|  \mathbf{v}_k^s \|^2  \big] \bigg).\nonumber
		\end{align}
	\end{subequations}
\end{lemma}
Another upper bound on SVRG's gradient estimate is available; see e.g., \citep{kulunchakov2019}, but it is not suitable for our analysis. Intuitively, Lemma \ref{lemma.est_err} suggests that if SVRG or SARAH converges, the MSE of their gradient estimates also approaches to zero. 

At the end of each inner loop, the starting point of the next outer loop is randomly selected among $\{ \mathbf{x}_k^s \}_{k=0}^m$ according to a pmf vector $\mathbf{p}^s \in 
\Delta_{m+1}$, where $\Delta_{m+1}:=\{ \mathbf{p} \in \mathbb{R}_+^{m+1} | \langle \bm{1}, \mathbf{p} \rangle = 1 \}$. We term $\mathbf{p}^s$ the \textit{averaging weight vector}, and let $p^s_j$ denote the $j$th entry of $\mathbf{p}^s$. Leveraging the MSE bounds in Lemma \ref{lemma.est_err} and choosing a proper averaging vector, SVRG and SARAH iterates for strongly convex problems can be proved to converge linearly. 

For SVRG, two types of averaging exist.
\begin{enumerate}
	\item[\textbullet] \textbf{U-Avg (SVRG)} \citep{johnson2013}: vector $\mathbf{p}^s$ is chosen as the pmf of an (almost) uniform distribution; that is, $p_m^s = 0$, and $p_k^s = 1/m$ for $k = \{0,1,\ldots,m-1\}$. Under Assumptions \ref{as.1} -- \ref{as.3}, the choice of $\eta = {\cal O}(1/L)$ and $m = {\cal O}(\kappa)$ ensures that SVRG iterates converge linearly.\footnote{For simplicity and clarity of exposition we only highlight the order of $\eta$ and $m$, and hide other constants in big-${\cal O}$ notation. Detailed choices can be found in the corresponding references.}
	\item[\textbullet] \textbf{L-Avg (SVRG)} \citep{tan2016,hu2018diss}: Only the last iteration is used for averaging by setting $\tilde{\mathbf{x}}^s = \mathbf{x}_m^s$; or equivalently, by setting $p_m^s = 1$, and $p_k^s = 0, \forall k \neq m$. Under Assumptions \ref{as.1} -- \ref{as.3}, linear convergence is ensured by choosing $\eta = {\cal O}(1/(L\kappa))$ and $m = {\cal O}(\kappa^2)$.
\end{enumerate}

To guarantee linear convergence, SVRG with L-Avg must adopt a much smaller $\eta$ and larger $m$ compared with U-Avg. L-Avg with such a small step size leads to complexity ${\cal O}\big((n+\kappa^2) \ln \frac{1}{\epsilon} \big)$ that has worse dependence on $\kappa$. 

For SARAH, there are also two averaging options.
\begin{enumerate}
	\item[\textbullet] \textbf{U-Avg (SARAH)} \citep{nguyen2017}: here $\bf{p}^s$ is selected to have entries $p_m^s = 0$, and $p_k^s = 1/m$, for $k = \{0,1,\ldots,m-1\}$. Linear convergence is guaranteed with complexity ${\cal O}\big((n+\kappa) \ln \frac{1}{\epsilon} \big)$ under Assumptions \ref{as.1} -- \ref{as.3} so long as one selects $\eta = {\cal O}(1/L)$ and $m = {\cal O}(\kappa)$.
	\item[\textbullet] \textbf{L-Avg (SARAH)} \citep{li2019l2s}\footnote{There is another version of L-Avg for SARAH \citep{liuclass}, but convergence claims require undesirably small step sizes $\eta = {\cal O}(\mu/L^2)$. This is why we focus on the L-Avg in \citep{li2019l2s}.}: here $\mathbf{p}^s$ is chosen with entries $p_{m-1}^s = 1$ and $p_k^s = 0, \forall k \neq m-1$. Under Assumptions \ref{as.1} -- \ref{as.3} and with $\eta = {\cal O}(1/L)$ as well as $m = {\cal O}(\kappa^2)$, linear convergence is guaranteed at complexity of ${\cal O}\big((n+\kappa^2) \ln \frac{1}{\epsilon} \big)$. When both Assumptions \ref{as.1} and \ref{as.4} hold, setting $\eta = {\cal O}(1/L)$ and $m = {\cal O}(\kappa)$ results in linear convergence along with a reduced complexity of order ${\cal O}\big((n+\kappa) \ln \frac{1}{\epsilon} \big)$. 
\end{enumerate}

U-Avg (for both SVRG and SARAH) is usually employed as a `proof-trick' to carry out convergence analysis, while L-Avg is implemented most of the times. However, we will argue in the next section that with U-Avg adapted to the step size choice, it is possible to improve empirical performance. Although U-Avg appears at first glance to waste updates, a simple trick in the implementation can fix this issue.

\textbf{Implementation of Averaging.} Rather than updating $m$ times and then choosing $\tilde{\mathbf{x}}^s$ according to Line 10 of SVRG or SARAH, one can generate a random integer $M^s \in \{0,1,\ldots,m \}$ according to the averaging weight vector $\mathbf{p}^s$. Having available $\mathbf{x}^s_{M^s}$, it is possible to start the next outer loop immediately.

\section{Weighted Averaging for SVRG and SARAH}
This section introduces weighted averaging for SVRG and SARAH, which serves as an intermediate step for the ultimate `tune-free variance reduction.' Such an averaging for SVRG will considerably tighten its analytical convergence rate; while for SARAH it will improve its convergence rate when $m$ or $\eta$ is chosen sufficiently large. These analytical results are obtained by reexamining SVRG and SARAH through the `estimate sequence' (ES), a tool that has been used for analyzing momentum schemes~\citep{nesterov2004}; see also \citep{nitanda2014,lin2015,kulunchakov2019}. Different from existing ES analysis that relies heavily on the unbiasedness of $\mathbf{v}_k^s$, our advances here will endow ES with the ability to deal with the biased gradient estimate of SARAH.


\subsection{Estimate Sequence}
Since in this section we will focus on a specific inner loop indexed by $s$, the superscript $s$ is dropped for brevity. For example, $\mathbf{x}_k^s$ and $\mathbf{v}_k^s$ are written as $\mathbf{x}_k$ and $\mathbf{v}_k$, respectively. 

Associated with the ERM objective $f$ and a particular point $\mathbf{x}_0$, consider a series of quadratic functions $\{ \Phi_k(\mathbf{x}) \}_{k=0}^m$ that comprise what is termed ES, with the first one given by 
\begin{subequations}\label{eq.est_seq}
\begin{align}
	\Phi_0 (\mathbf{x}) = \Phi_0^* + \frac{\mu_0}{2} \| \mathbf{x} -  \mathbf{x}_0 \|^2	
\end{align}
and the rest defined recursively as
\begin{align}
	 \Phi_k (\mathbf{x})  = &(1- \delta_k)\Phi_{k-1} (\mathbf{x}) + \delta_k \Big[  f(\mathbf{x}_{k-1})  + \langle  \mathbf{v}_{k-1}, \mathbf{x} - \mathbf{x}_{k-1}  \rangle + \frac{\mu}{2} \| \mathbf{x} -  \mathbf{x}_{k-1} \|^2	\Big] \nonumber
\end{align}
\end{subequations}
where $\mathbf{v}_{k-1}$ is the gradient estimate in SVRG or SARAH; while $\Phi_0^*$, $\mu_0$, and $\delta_k$ are some constants to be specified later. The design is similar to that of \citep{kulunchakov2019}, but the ES here is constructed per inner loop. In addition, here we will overcome the challenge of analyzing SARAH's biased gradient estimate $\mathbf{v}_k$. 

Upon defining $\Phi_k^* := \min_{\mathbf{x}} \Phi_k(\mathbf{x})$, the key properties of the sequence 
$\{ \Phi_k(\mathbf{x}) \}_{k=0}^m$ are collected in the next lemma.

\begin{lemma}\label{lemma.est_seq}
For $\{ \Phi_k(\mathbf{x}) \}_{k=0}^m$ as in \eqref{eq.est_seq}, it holds that: i) $\Phi_0(\mathbf{x})$ is $\mu_0$-strongly convex, and $\Phi_k(\mathbf{x})$ is $\mu_k$-strongly convex with $\mu_k = (1-\delta_k) \mu_{k-1} + \delta_k \mu$; ii) $\mathbf{x}_k$ minimizes $\Phi_k(\mathbf{x})$ if $\delta_k = \eta \mu_k$; and iii) $\Phi_k^* = (1 - \delta_k )\Phi_{k-1}^* + \delta_k f(\mathbf{x}_{k-1}) - \frac{ \mu_k \eta^2 }{2} \| \mathbf{v}_{k-1} \|^2$.
\end{lemma}
Lemma \ref{lemma.est_seq} holds for both SVRG and SARAH. To better understand the role of ES, it is instructive to use an example.

\noindent\textbf{Example.} With $\Phi_0^* = f(\mathbf{x}_0)$, $\mu_0=\mu$, and $\delta_k = \mu_k \eta$ for SVRG, it holds that $\mu_k = \mu, \forall k$, and $\delta_k = \mu\eta, \forall k$. If for convenience we let $\delta:= \mu\eta$, we show in Appendix \ref{apdx.eg} that
\begin{align}\label{eq.svrgeg}
	\mathbb{E}\big[ \Phi_k (\mathbf{x}) \big] \leq (1-\delta)^k \big[\Phi_0(\mathbf{x}) - f(\mathbf{x}^*)\big] + f(\mathbf{x}).
\end{align}
As $k \rightarrow \infty$, one has $(1-\delta)^k \rightarrow 0$, and hence $ \Phi_k (\mathbf{x}) $ approaches in expectation a lower bound of $f(\mathbf{x})$. 

Now, we are ready to view SVRG and SARAH through the lens of $\{ \Phi_k(\mathbf{x}) \}_{k=0}^m$ to obtain new averaging schemes.

\subsection{Weighted Averaging for SVRG}
The new averaging vector $\mathbf{p}^s$ for SVRG together with the improved convergence rate is summarized in the following theorem.

\begin{theorem}\label{thm.svrg}
\textbf{(SVRG with W-Avg.)}
Under Assumptions \ref{as.1} -- \ref{as.3}, construct the ES as in \eqref{eq.est_seq} with $\mu_0 = \mu$, $\delta_k = \mu_k\eta$, and $\Phi_0^* = f(\mathbf{x}_0)$. Choose $\eta < 1/(4L)$, and $m$ large enough such that 
	\begin{align*}
		\lambda^{\texttt{SVRG}} : =  &\frac{1}{1 - (1 - \mu \eta)^{m-1}} \bigg[ \frac{(1 - \mu \eta)^m}{ 1 - 2 \eta L } \nonumber \\
	& + \frac{2 \mu L \eta^2   (1 - \mu \eta )^{m-1}}{1 - 2  L \eta } +  \frac{2  L \eta }{1 - 2 L \eta } \bigg] < 1.
	\end{align*}
Let $p_0^s = p_m^s = 0$, and $p_k^s = (1 - \mu\eta)^{m-k -1}/q$ for $k = 1,2,\ldots, m-1$, where $q = [1 - (1 - \mu\eta)^{m-1}]/(\mu\eta)$. It then holds for SVRG with this weighted averaging (W-Avg) that
	\begin{align*}
	\mathbb{E} \big[ f(\tilde{\mathbf{x}}^{s}) - f(\mathbf{x}^*) \big]  \leq \lambda^{\texttt{SVRG}} \mathbb{E}\big[ f(\tilde{\mathbf{x}}^{s-1}) - f(\mathbf{x}^*)\big].
	\end{align*}
\end{theorem}
Comparing the W-Avg in Theorem \ref{thm.svrg} against U-Avg and L-Avg we saw in Section \ref{sec.recap}, the upshot of W-Avg is a much tighter convergence rate. When choosing $\eta = {\cal O}(1/L)$, the dominating terms of the convergence rate for W-Avg are ${\cal O}\big(\frac{(1-1/\kappa)^m}{1-2L\eta} +\frac{2L\eta }{1-2L\eta}\big)$, and ${\cal O}\big(\frac{\kappa}{m (1-2L\eta)} +\frac{2L\eta }{1-2L\eta}\big)$ for U-Avg \citep{johnson2013}. Clearly, the factor $(1-1/\kappa)^m$ in W-Avg can be much smaller than $\kappa/m$ in U-Avg; see Fig. \ref{fig.rate_svrg_sarah}(a) for comparison of convergence rates of different averaging types. Since convergence of SVRG with L-Avg requires $\eta$ and $m$ to be chosen differently from those in U-Avg and W-Avg, L-Avg is not plotted in Fig. \ref{fig.rate_svrg_sarah}(a).

Next, we assess the complexity of SVRG with W-Avg.

\begin{corollary}\label{coro.svrg}
	Choosing $m = {\cal O}(\kappa)$ and other parameters as in Theorem \ref{thm.svrg}, the complexity of SVRG with W-Avg to find $\tilde{\mathbf{x}}^s$ satisfying $\mathbb{E}\big[ f(\tilde{\mathbf{x}}^s) - f(\mathbf{x}^*)\big] \leq \epsilon $ is ${\cal O}\big( (n+\kappa) \ln\frac{1}{\epsilon} \big)$.
\end{corollary}
Note that similar to U-Avg, W-Avg incurs lower complexity compared with L-Avg in \citep{tan2016,hu2018diss}.

\subsection{Weighted Averaging for SARAH} 

\begin{figure}[t]
	\centering
	\begin{tabular}{cc}
		\hspace{-0.2cm}
		\includegraphics[width=.40\textwidth]{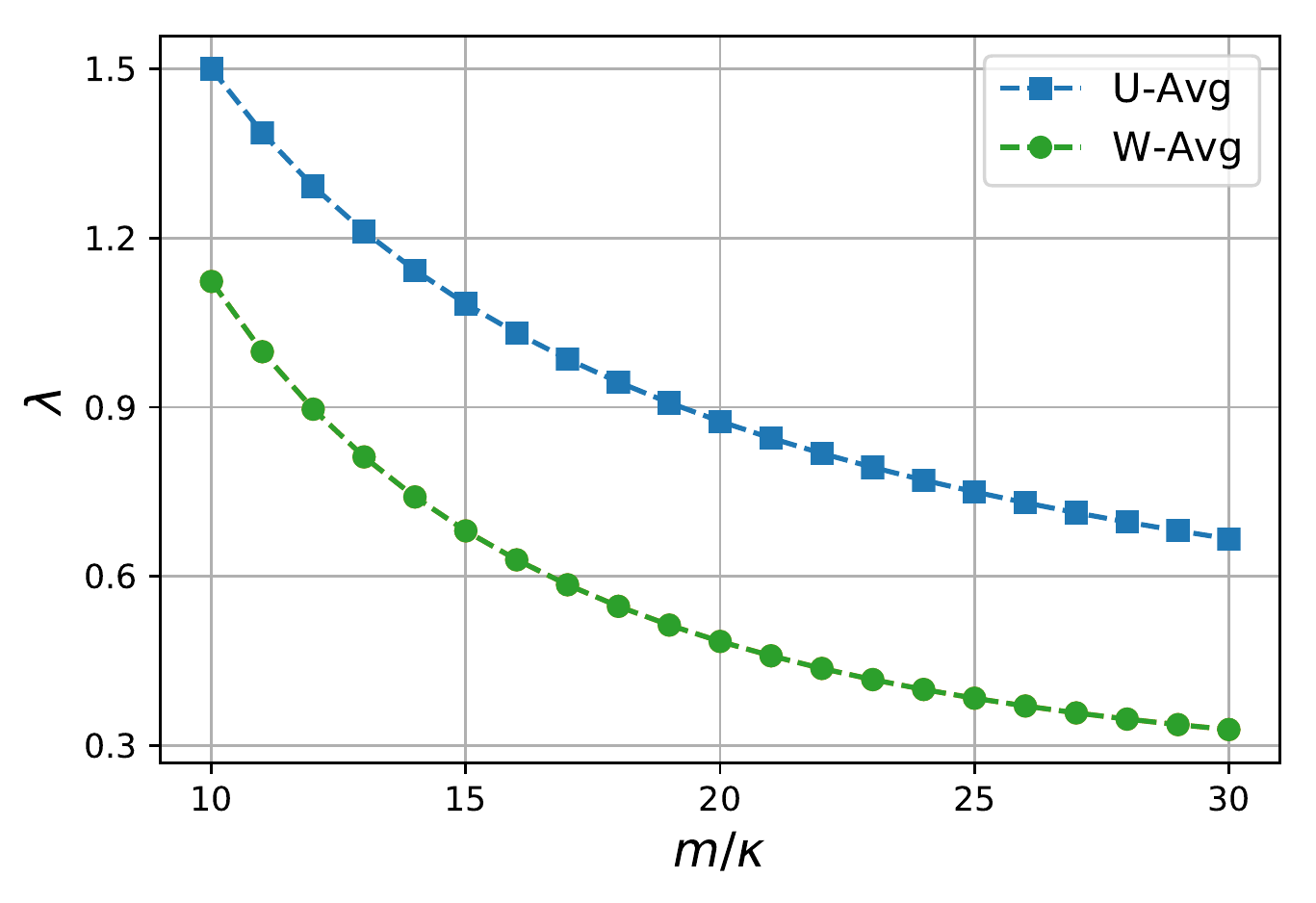}&
		\hspace{-0.1cm}
		\includegraphics[width=.40\textwidth]{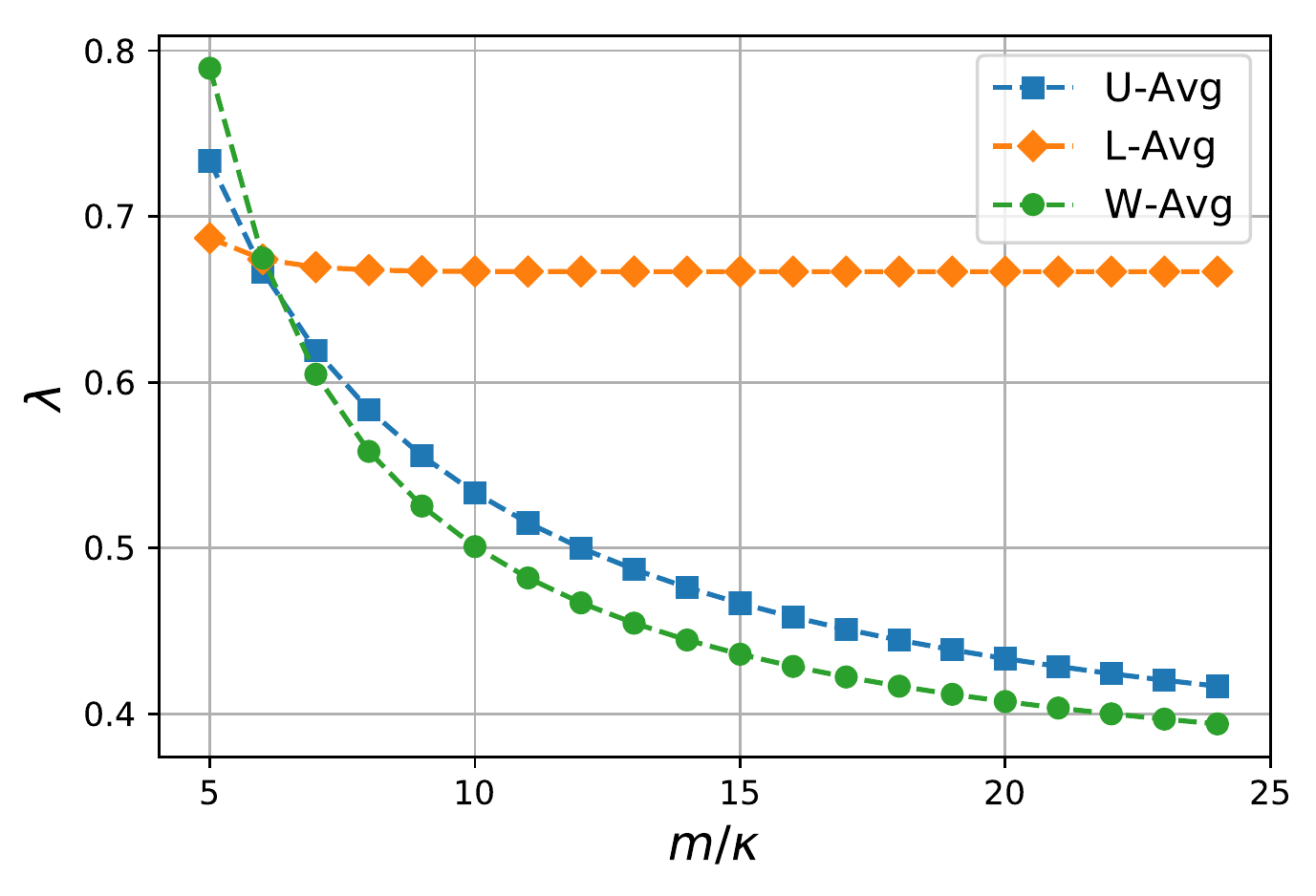}
		\\ (a) SVRG  & (b) SARAH
	\end{tabular}
	\caption{A comparison of the analytical convergence rate for SVRG and SARAH. In both figures we set $\kappa = 10^5$ with $L =1$, $\mu = 10^{-5}$, and the step sizes are selected as: (a) SVRG with $\eta = 0.1/L$; and (b) SARAH with $\eta = 0.5/L$. }
	 \label{fig.rate_svrg_sarah}
\end{figure}

SARAH is challenging to analyze due to the bias present in the estimate $\mathbf{v}_k$, which makes the ES-based treatment of SARAH fundamentally different from that of SVRG. To see this, it is useful to start with the following lemma.

\begin{lemma}\label{lemma.sarah_addition}
For any deterministic $\mathbf{x}$, it holds in SARAH that
	\begin{align*}
		& \mathbb{E}\big[ \langle \mathbf{v}_k - \nabla f(\mathbf{x}_k), \mathbf{x} - \mathbf{x}_k \rangle \big] = \frac{\eta}{2} \sum_{\tau=0}^{k-1} \mathbb{E}\Big[ \| \mathbf{v}_\tau - \nabla f(\mathbf{x}_\tau) \|^2 + \| \mathbf{v}_\tau  \|^2 - \| \nabla f(\mathbf{x}_\tau ) \|^2 \Big].\nonumber
	\end{align*}
\end{lemma}
Lemma \ref{lemma.sarah_addition} reveals the main difference in the ES-based argument for SARAH, namely that $\mathbb{E}\big[ \langle \mathbf{v}_k - \nabla f(\mathbf{x}_k), \mathbf{x} - \mathbf{x}_k \rangle \big] \neq 0$, while the same inner product for SVRG equals to $0$ in expectation. Reflecting back to \eqref{eq.svrgeg}, the consequence of having a non-zero $\mathbb{E}\big[ \langle \mathbf{v}_k - \nabla f(\mathbf{x}_k), \mathbf{x} - \mathbf{x}_k \rangle \big] $ is that $\mathbb{E}[ \Phi_k (\mathbf{x}) ]$ is not necessarily approaching a lower bound of $f(\mathbf{x})$ as $k \rightarrow \infty$; thus,
\begin{align}\label{eq.saraheg}
	\mathbb{E}\big[ \Phi_k (\mathbf{x}) \big]
	 \leq (1-\delta)^k \big[\Phi_0(\mathbf{x}) - f(\mathbf{x})\big] + f(\mathbf{x}) + C 
\end{align}
where $C$ is a non-zero term that is not present in \eqref{eq.svrgeg} when applied to SVRG; see detailed derivations in Appendix \ref{apdx.eg}.

Interestingly, upon capitalizing on the properties of $\mathbf{v}_k$, the ensuing theorem establishes linear convergence for SARAH with a proper W-Avg vector $\mathbf{p}^s$. 
\begin{theorem}\label{thm.sarah}
\textbf{(SARAH with W-Avg.)}
Under Assumptions \ref{as.1} and \ref{as.4}, define the ES as in \eqref{eq.est_seq} with $\mu_0 = \mu$, $\delta_k = \mu_k \eta, \forall k$, and $\Phi_0^* = f(\mathbf{x}_0)$. With $\delta:= \mu\eta$, select $\eta < 1/L$ and $m$ large enough, so that
	\begin{align*}
		\lambda^{\texttt{SARAH}} & := \bigg[ (1\! -\! \delta)^m -\Big( 1 \!-\! \frac{2\eta L}{1\!+\! \kappa} \Big)^m \bigg] \frac{ L\! +\! \mu}{ c(L\! -\! \mu)} + \frac{ (1\! -\!  \delta)^m }{c \delta} 
		 + \frac{\eta L(m \!-\! 1)}{c(2 \!-\! \eta L)} +  \frac{2 \!-\! 2 \eta L}{2\! -\! \eta L} \frac{1+\kappa}{2 c \eta L}    < 1
	\end{align*}
where $c = m - \frac{1}{\delta} + \frac{(1-\delta)^m}{\delta}$. Setting $p_k = (1 - (1-\delta)^{m-k-1})/c, \forall k = 0,1,\ldots, m-2$, and $p_{m-1} = p_m = 0$, SARAH with this W-Avg satisfy 
	\begin{align*}
	\mathbb{E} \big[ \|\nabla f( \tilde{\mathbf{x}}^s ) \|^2 \big]  \leq \lambda^{\texttt{SARAH}} \mathbb{E} \big[ \|\nabla f( \tilde{\mathbf{x}}^{s-1} ) \|^2 \big].
\end{align*}
\end{theorem}
The expression of $\lambda^{\texttt{SARAH}}$ is complicated because we want the upper bound of the convergence rate to be as tight as possible. To demonstrate this with an example, choosing $\eta = 1/(2L)$ and $m = 5 \kappa$, we have $\lambda^{\texttt{SARAH}}\approx 0.8$. Fig. \ref{fig.rate_svrg_sarah}(b) compares SARAH with W-Avg versus SARAH with U-Avg and L-Avg. The advantage of W-Avg is more pronounced as $m$ is chosen larger. 

As far as complexity of SARAH with W-Avg, it is comparable with that of L-Avg or U-Avg, as asserted next.   

\begin{corollary}\label{coro.sarah}
	Choosing $m = {\cal O}(\kappa)$ and other parameters as in Theorem \ref{thm.sarah}, the complexity of SARAH with W-Avg to find $\tilde{\mathbf{x}}^s$ satisfying $\mathbb{E} [ \| \nabla f(\tilde{\mathbf{x}}^s)\|^2] \leq \epsilon $, is ${\cal O}\big( (n+\kappa) \ln\frac{1}{\epsilon} \big)$.
\end{corollary}
A few remarks are now in order on our analytical findings: i) most existing ES-based proofs use 
$\mathbb{E}[f(\tilde{\mathbf{x}}^s) - f(\mathbf{x}^*)]$ as optimality metric, while Theorem \ref{thm.sarah} and Corollary \ref{coro.sarah} rely on $\mathbb{E}[\| \nabla f(\tilde{\mathbf{x}}^s)\|^2]$; ii) the analysis method still holds when Assumption \ref{as.4} is weakened to Assumption \ref{as.3}, at the price of having  worse $\kappa$-dependence of the complexity; that is, ${\cal O}\big( (n+\kappa^2) \ln\frac{1}{\epsilon}\big)$, which is of the same order as L-Avg under Assumptions \ref{as.1} -- \ref{as.3} \citep{li2019l2s,liuclass}.


\subsection{Averaging Is More Than A `Proof Trick'}\label{sec.avg}

\begin{wrapfigure}{l}{0.45\textwidth}
\vspace{-0.3cm}
	\centering
	\includegraphics[height=3.8cm]{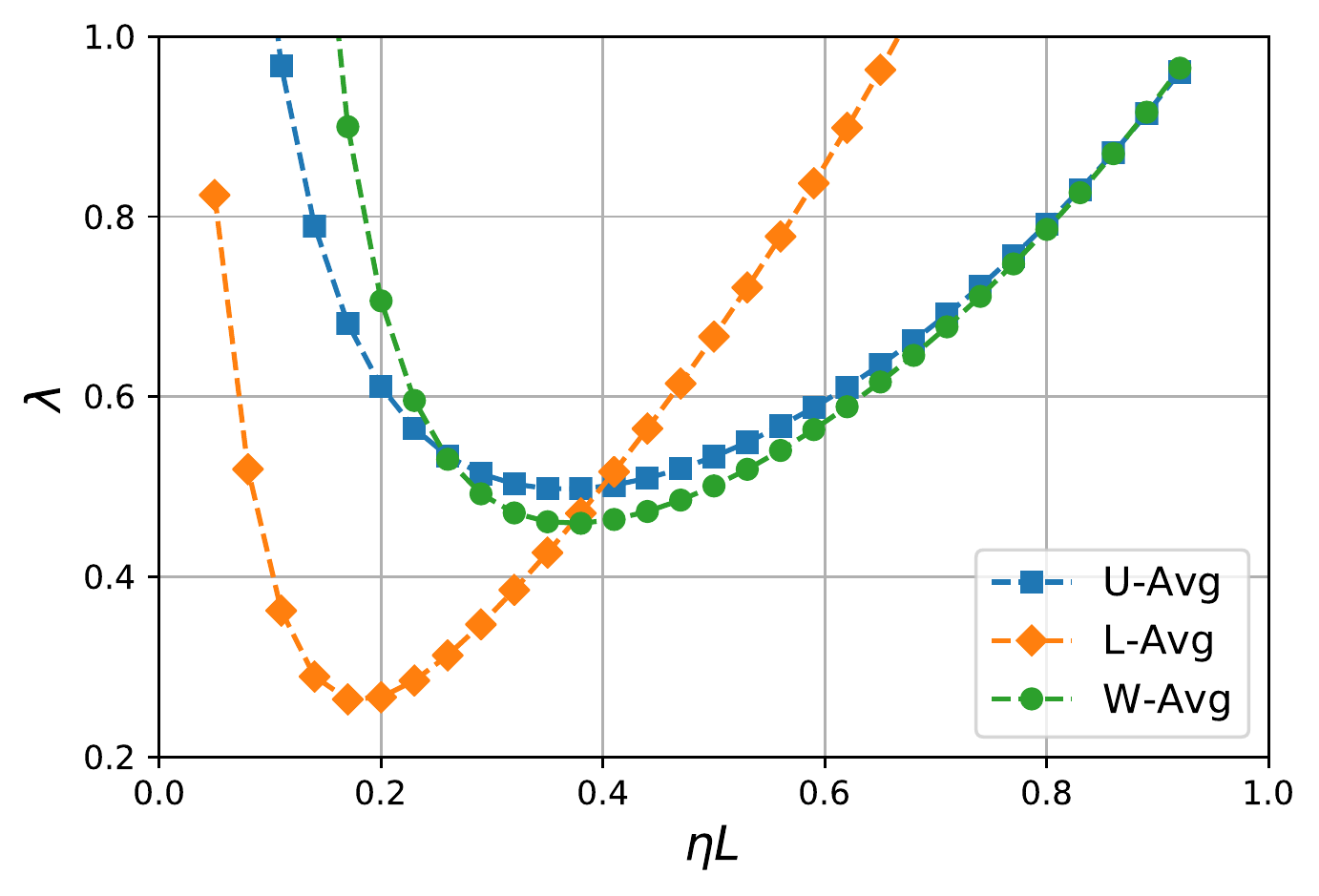}
	\vspace{-0.6cm}
	\caption{SARAH's analytical convergence with different averaging options ($\kappa = 10^5$, $L =1$, $\mu = 10^{-5}$, and fixed $m = 10\kappa$).}  
	\label{fig.sarah_theo1}
	\vspace{-0.3cm}
\end{wrapfigure}

Existing forms of averaging such as U-Avg and W-Avg, are typically considered as `proof tricks' for simplifying the theoretical analysis \citep{johnson2013,tan2016,nguyen2017,li2019l2s}. In this subsection, we contend that averaging can distinctly affect performance, and should be adapted to other parameters. We will take SARAH with $\eta = {\cal O}(1/L)$ and $m = {\cal O}(\kappa)$ as an example rather than SVRG since such parameter choices guarantee convergence regardless of the averaging employed. (For SVRG with L-Avg on the other hand, the step size has to be chosen differently with W-Avg or U-Avg.)

We will first look at the convergence rate of SARAH across different averaging options. Fixing $m = {\cal O}(\kappa)$ and changing $\eta$, the theoretical convergence rate is plotted in Fig. \ref{fig.sarah_theo1}. It is observed that with smaller step sizes, L-Avg enjoys faster convergence, while larger step sizes tend to favor W-Avg and U-Avg instead.

\begin{figure*}[t]
	\centering
	\begin{tabular}{ccc}
		\hspace{-0.2cm}
		\includegraphics[width=.32\textwidth]{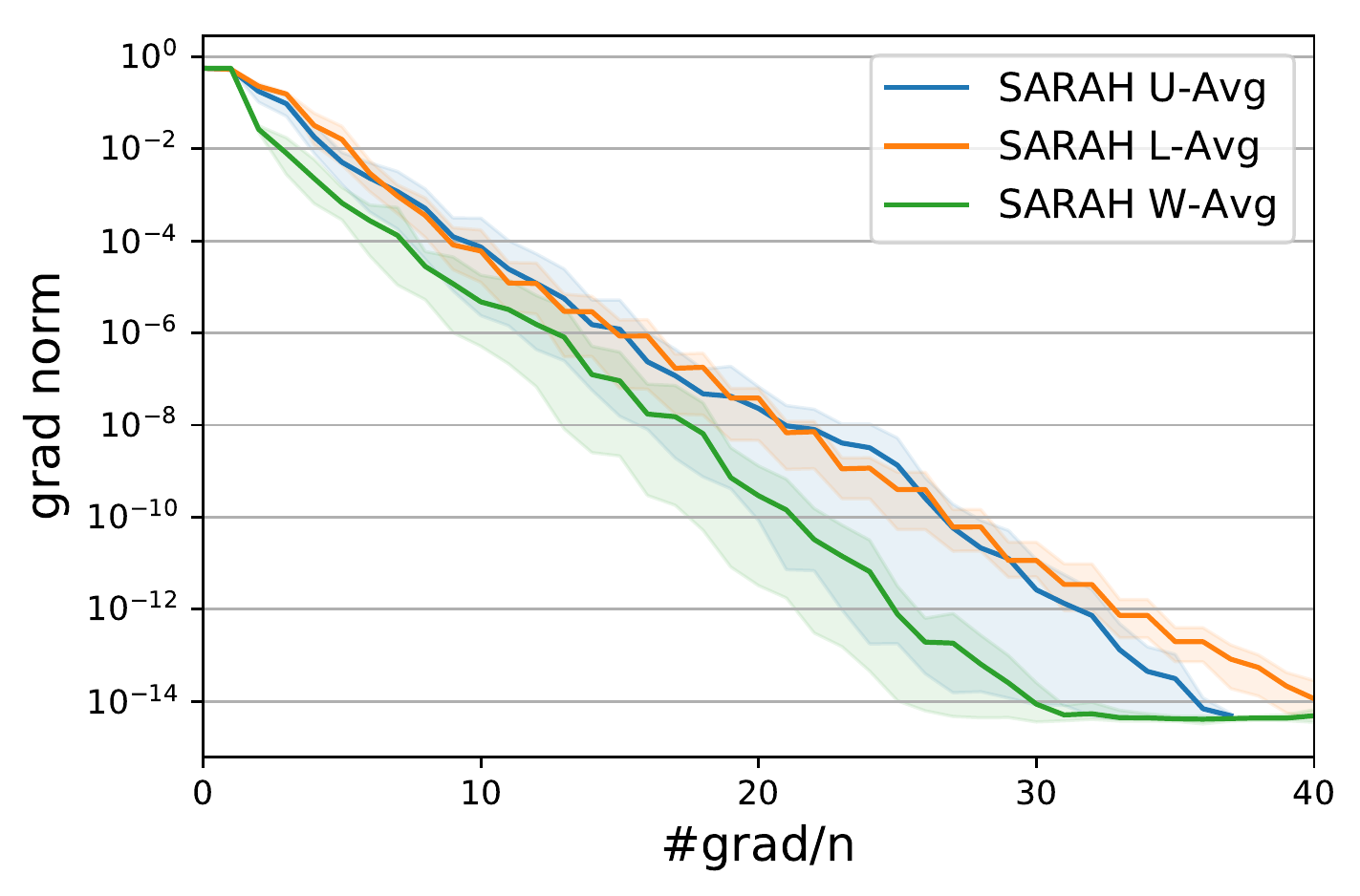}&
		\hspace{-0.2cm}
		\includegraphics[width=.32\textwidth]{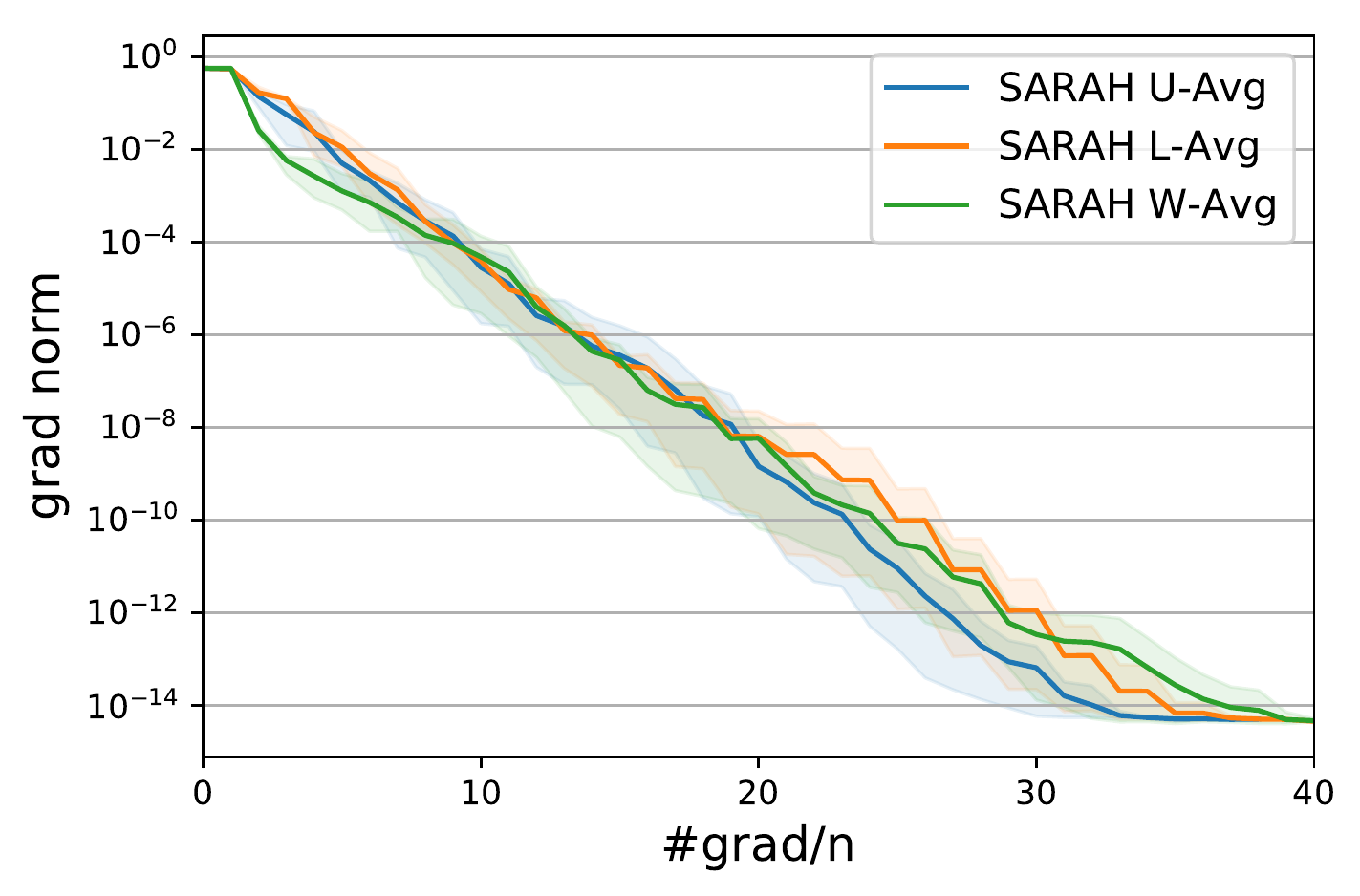}&
		\hspace{-0.3cm}
		\includegraphics[width=.32\textwidth]{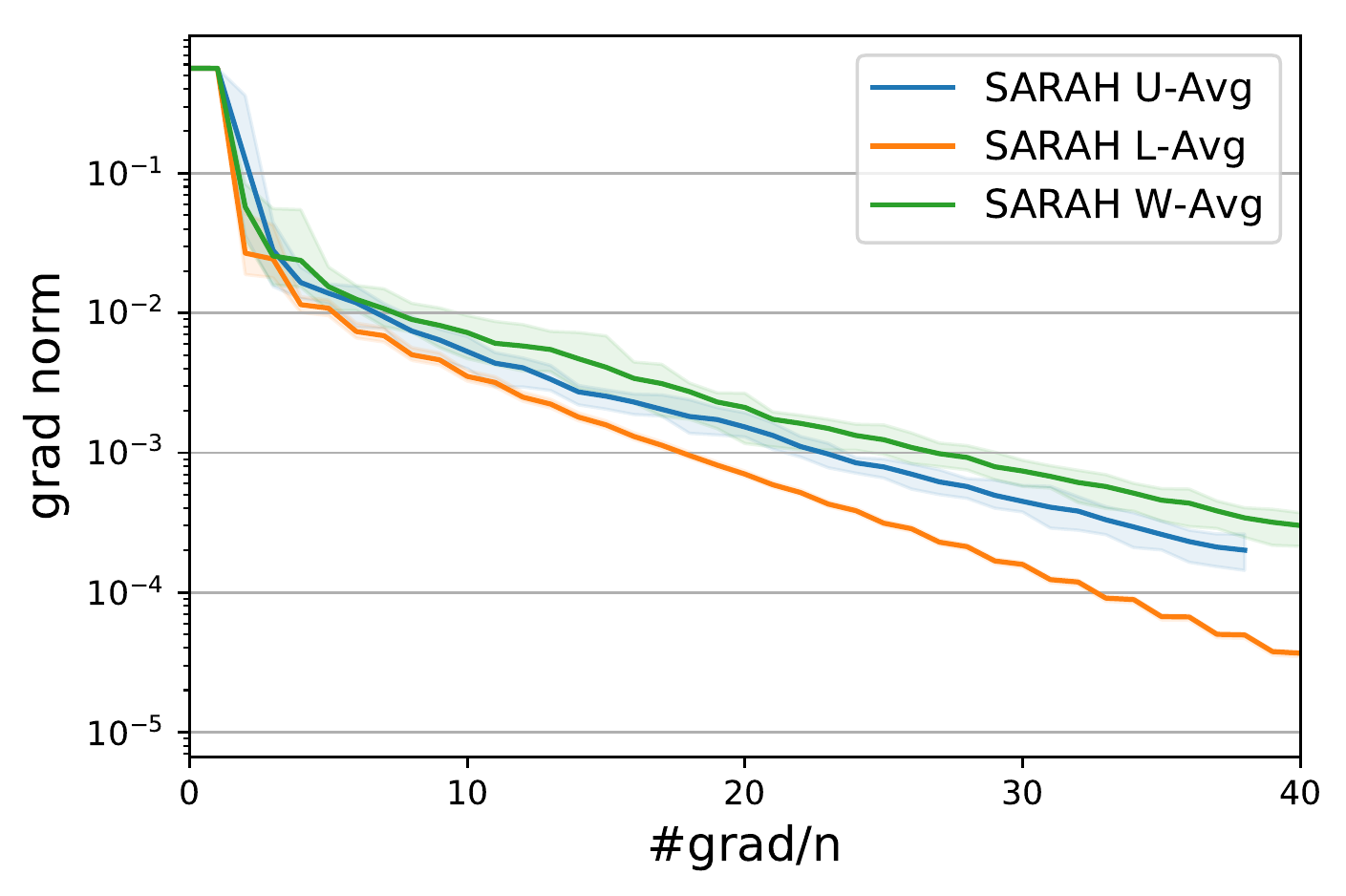}
		\\ (a) $\eta = 0.9/L$ & (b) $\eta = 0.6/L$ & (c) $\eta = 0.06/L$ 
	\end{tabular}
	\caption{Comparing SARAH with different types of averaging on dataset \textit{w7a} ($\mu = 0.005$ and $m = 5\kappa$ in all tests). }
	 \label{fig.w7a}
\end{figure*}

Next, we will demonstrate empirically that the type of averaging indeed matters. Consider binary classification using the regularized logistic loss function
\begin{equation}\label{eq.test}
	f(\mathbf{x}) =\frac{1}{n} \sum_{i \in [n]} \ln \big[1+ \exp(- b_i \langle \mathbf{a}_i, \mathbf{x} \rangle ) \big] + \frac{\mu}{2} \| \mathbf{x} \|^2
\end{equation}
where $(\mathbf{a}_i, b_i)$ is the (feature, label) pair of datum $i$. Clearly, \eqref{eq.test} is an instance of the cost in \eqref{eq.prob} with $f_i(\mathbf{x}) =  \ln \big[1+ \exp(- b_i \langle \mathbf{a}_i, \mathbf{x} \rangle ) \big] + \frac{\mu}{2} \| \mathbf{x} \|^2$; and it can be readily verified that Assumptions \ref{as.1} and \ref{as.4} are satisfied in this case.

SARAH with L-Avg, U-Avg and W-Avg are tested with fixed (moderate) $m = {\cal O}(\kappa)$ but different step size choices on the dataset \textit{w7a}; see also Appendix \ref{appdx.avg_tests} for additional tests with datasets \textit{a9a} and \textit{diabetes}.  Fig. \ref{fig.w7a}(a) shows that for a large step size $\eta = 0.9/L$, W-Avg outperforms U-Avg as well as L-Avg by almost two orders at the $30$th sample pass. For a medium step size $\eta = 0.6/L$, W-Avg and L-Avg perform comparably, while both are outperformed by U-Avg. When $\eta$ is chosen small, L-Avg is clearly the winner. In short, the performance of averaging options varies with the step sizes. This is intuitively reasonable because the MSE of $\mathbf{v}_k$: i) scales with $\eta$ (cf. Lemma \ref{lemma.est_err}); and ii) tends to increase with $k$ as $\mathbb{E}[\|\mathbf{v}_k\|^2]$ decreases linearly (see Lemma \ref{lemma.sarah_v_norm} in Appendix \ref{appdx.pf_sarah}, and the MSE bound in Lemma \ref{lemma.est_err}). As a result, when both $\eta$ and $k$ are large, the MSE of $\mathbf{v}_k$ tends to be large too. Iterates with gradient estimates having high MSE can jeopardize the convergence. This explains the inferior performance of L-Avg in Figs. \ref{fig.w7a}(a) and \ref{fig.w7a}(b). On the other hand, when $\eta$ is chosen small, the MSE tends to be small as well; hence, working with L-Avg does not compromise convergence, while in expectation W-Avg and U-Avg compute full gradient more frequently than L-Avg. These two reasons explain the improved performance of L-Avg in Fig. \ref{fig.w7a}(c).

When we fix $\eta$ and change $m$, as depicted in Fig. \ref{fig.rate_svrg_sarah}(b), the analytical convergence rate of W-Avg improves over that of U-Avg and L-Avg when $m$ is large. This is because the MSE of $\mathbf{v}_k$ increases with $k$. W-Avg and U-Avg ensure better performance through ``early ending,'' by reducing the number of updates that utilize $\mathbf{v}_k$ with large MSE.

In sum, the choice of averaging scheme should be adapted with $\eta$ and $m$ to optimize performance. For example, the proposed W-Avg for SARAH favors the regime where either $\eta$ or $m$ is chosen large, as dictated by the convergence rates and corroborated by numerical tests.

\section{Tune-Free Variance Reduction}
This section copes with variance reduction without tuning. In particular, i) Barzilai-Borwein (BB) step size, ii) averaging schemes, and iii) a time varying inner loop length are adopted for the best empirical performance.

\subsection{Recap of BB Step Sizes}
Aiming to develop `tune-free' SVRG and SARAH, we will first adopt the BB scheme to obtain suitable step sizes automatically~\citep{tan2016}. In a nutshell, BB monitors progress of previous outer loops, and chooses the step size of outer loop $s$ accordingly via
\begin{equation}\label{eq.bb_stepsize}
	\eta^s = \frac{1}{\theta_\kappa} \frac{\| \tilde{\mathbf{x}}^{s-1} - \tilde{\mathbf{x}}^{s-2} \|^2}{ \big\langle \tilde{\mathbf{x}}^{s-1} - \tilde{\mathbf{x}}^{s-2}, \nabla f(\tilde{\mathbf{x}}^{s-1}) -  \nabla f(\tilde{\mathbf{x}}^{s-2}) \big\rangle}
\end{equation}
where $\theta_\kappa$ is a $\kappa$-dependent parameter to be specified later. Note that $\nabla f(\tilde{\mathbf{x}}^{s-1})$ and $\nabla f(\tilde{\mathbf{x}}^{s-2})$ are computed at the outer loops $s$ and $s-1$, respectively; hence, the implementation overhead of BB step sizes only includes almost negligible memory to store $\tilde{\mathbf{x}}^{s-2}$ and $\nabla f(\tilde{\mathbf{x}}^{s-2})$.

BB step sizes for SVRG with L-Avg have relied on $\theta_\kappa = m = {\cal O}(\kappa^2)$~\citep{tan2016}. Such a choice of parameters offers provable convergence at complexity ${\cal O}\big( (n + \kappa^2) \ln \frac{1}{\epsilon} \big)$, but has not been effective in our simulations for two reasons: i) step size $\eta^s$ depends on $m$, which means that tuning is still required for step sizes; and, ii) the optimal $m$ of ${\cal O}(\kappa)$ with best empirical performance significantly deviates from the theoretically suggested ${\cal O}(\kappa^2)$; see also Fig. \ref{fig.m_choices}(a). Other BB based variance reduction methods introduce extra parameters to be tuned in additional to $m$; e.g., some scalars in \citep{liuclass} and the mini-batch sizes in \citep{yang2019}. This prompts us to design more practical BB methods -- how to choose $m$ with minimal tuning is also of major practical value. 

\begin{figure}[t]
	\centering
	\begin{tabular}{cc}
		\hspace{-0.2cm}
		\includegraphics[width=.40\textwidth]{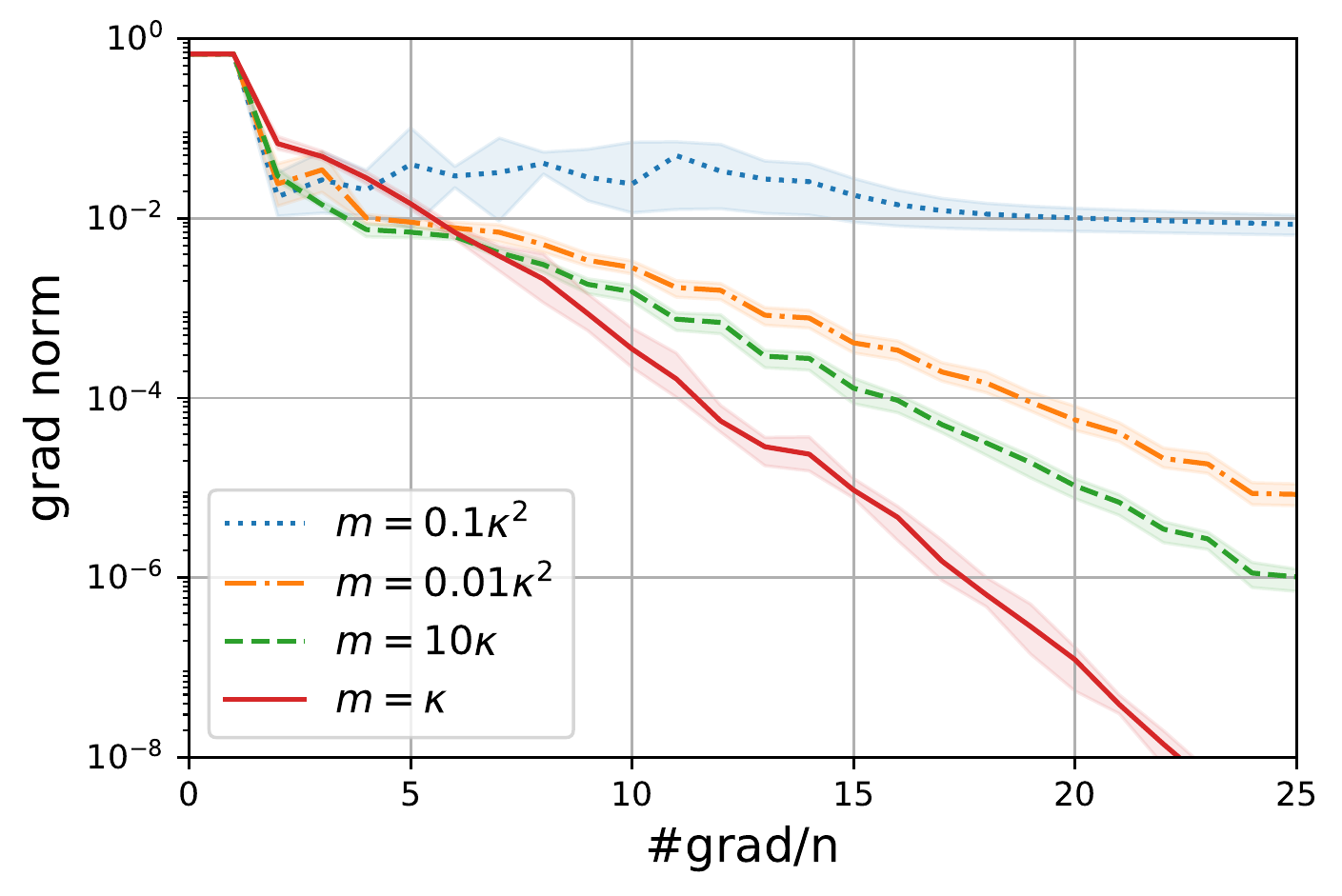}&
		\hspace{-0.1cm}
		\includegraphics[width=.40\textwidth]{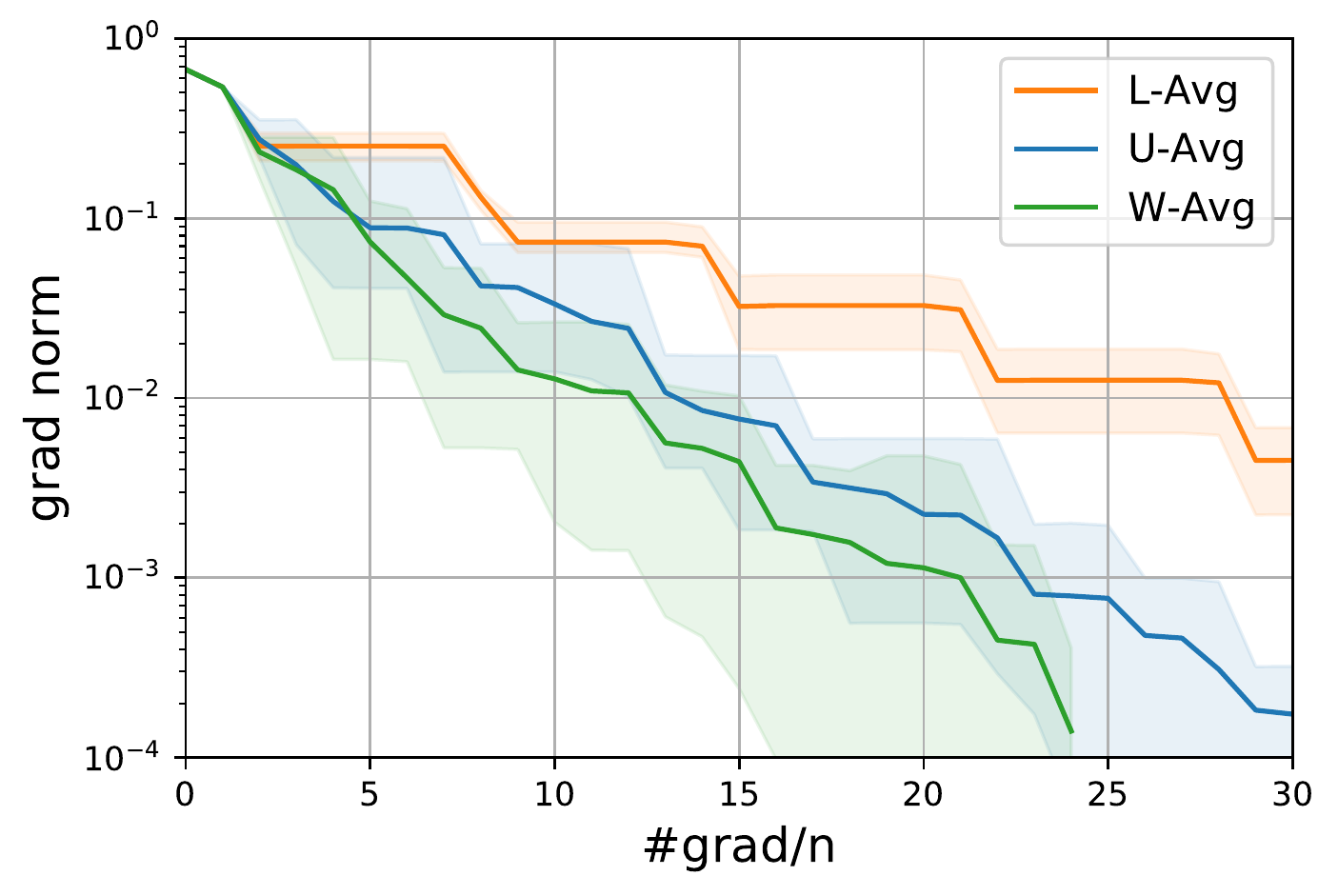}
		\\ (a)  & (b)
	\end{tabular}
	\caption{(a) Performance of BB-SVRG \citep{tan2016} under different choices of $m$. (b) Performance of BB-SARAH with different averaging schemes. Both tests use dataset \textit{a9a} with $\kappa = 1,388$.}
	 \label{fig.m_choices} 
\end{figure}

\subsection{Averaging for BB Step Sizes}
We start with a fixed choice of $m$ to theoretically investigate different types of averaging for the BB step sizes. The final `tune-free' implementation of SVRG and SARAH will rely on the analysis of this subsection.

\begin{proposition}\label{prop.bbsvrg}
\textbf{(BB-SVRG)} Under Assumptions \ref{as.1} -- \ref{as.3}, if we choose $m = {\cal O}(\kappa^2)$ and $\theta_\kappa = {\cal O}(\kappa)$ (but with $\theta_\kappa > 4\kappa$), then BB-SVRG with U-Avg and W-avg can find $\tilde{\mathbf{x}}^s$ with $\mathbb{E}\big[ f(\tilde{\mathbf{x}}^s) - f(\mathbf{x}^*) \big] \leq \epsilon$ using ${\cal O}\big( (n+\kappa^2)\ln\frac{1}{\epsilon} \big)$ IFO calls.
\end{proposition}

Similar to BB-SVRG, the ensuing result asserts that for BB-SARAH, W-Avg, U-Avg and L-Avg have identical order of complexity.
\begin{proposition}\label{prop.bbsarah}
\textbf{(BB-SARAH)} Under Assumptions \ref{as.1} and \ref{as.4}, if we choose $m = {\cal O}(\kappa^2)$ and $\theta_\kappa = {\cal O}(\kappa)$, then BB-SARAH finds a solution with $\mathbb{E}\big[ \| \nabla f(\tilde{\mathbf{x}}^s) \|^2 \big] \leq \epsilon$ using ${\cal O}\big( (n+\kappa^2)\ln\frac{1}{\epsilon} \big)$ IFO calls, when one of these conditions holds: i) either U-Avg with $\theta_\kappa > \kappa$; or ii) L-Avg with $\theta_\kappa > 3/2\kappa$; or, iii) W-Avg with $\theta_\kappa > \kappa$.
\end{proposition}

The price paid for having automatically tuned step sizes is a worse dependence of the complexity on $\kappa$, compared with the bounds in Corollaries \ref{coro.svrg} and \ref{coro.sarah}. The cause of the worse dependence on $\kappa$ is that one has to choose a large $m$ at the order of ${\cal O}(\kappa^2)$. However, such an automatic tuning of the step size comes almost as a ``free lunch'' when problem \eqref{eq.prob} is well conditioned, or, in the big data regime, e.g., ${\kappa^2 \approx n}$ or ${\kappa^2 \ll n}$, since the dominant term in complexity is ${\cal O}(n\ln\frac{1}{\epsilon})$ for both SVRG and BB-SVRG. On the other hand, it is prudent to stress that with $\kappa^2 \gg n$, the BB step sizes slow down convergence.

Given the same order of complexity, the empirical performance of BB-SARAH with different averaging types is showcased in Fig. \ref{fig.m_choices}(b) with the parameters chosen as in Proposition \ref{prop.bbsarah}. It is observed that W-Avg converges most rapidly, while U-Avg outperforms L-Avg. This confirms our theoretical insight, that is, W-Avg and U-Avg are more suitable when $m$ is chosen large enough.

\subsection{Tune-Free Variance Reduction}

Next, the ultimate format of the almost tune-free variance reduction is presented using SARAH as an example. We will discuss how to choose the iteration number of inner loops and averaging schemes for BB step sizes.


\textbf{Adaptive inner loop length.}
It is observed that the BB step size can change over a wide range of values (see Appendix \ref{appdx.bb_pf} for derivations), 
\begin{align}\label{eq.eta_bounds}
	\frac{1}{\theta_\kappa L} \leq \eta^s \leq \frac{1}{\theta_\kappa \mu}.
\end{align} 
Given $\theta_\kappa = {\cal O}(\kappa)$, $\eta^s$ can vary from ${\cal O}(\mu/L^2)$ to ${\cal O}(1/L)$. Such a wide range of $\eta^s$ blocks the possibility to find a single $m$ suitable for both small and large $\eta^s$ at the same time. From a theoretical perspective, choosing $m = {\cal O}(\kappa^2)$ in both Propositions \ref{prop.bbsvrg} and \ref{prop.bbsarah} is mainly for coping with the small step sizes $\eta^s = {\cal O}(1/(L \theta_\kappa))$. But such a choice is too pessimistic for large ones $\eta^s= {\cal O}(1/(\mu \theta_\kappa))$. In fact, choosing $m = {\cal O}(\kappa)$ for $\eta^s = {\cal O}(1/L)$ is good enough, as suggested by Corollaries \ref{coro.svrg} and \ref{coro.sarah}. These observations motivate us to design an $m^s$ that changes dynamically per outer loop $s$.

Reflecting on the convergence of SARAH, it is sufficient to set the inner loop length $m^s$ according to the $\eta^s$ used. To highlight the rationale behind our choice of $m^s$, let us consider BB-SARAH with U-Avg as an example that features convergence rate $\lambda^s = \frac{1}{\mu \eta^s m^s} + \frac{\eta^s L}{2 - \eta^s L}$ \citep{nguyen2017}. Set $\theta_\kappa > \kappa$ as in Proposition \ref{prop.bbsarah} so that the second term of $\lambda^s$ is always less than $1$. With a large step size $\eta^s = {\cal O}(1/L)$, and by simply choosing $m^s = {\cal O}\big(1/(\mu \eta^s) \big)$, one can ensure a convergent iteration having e.g., $\lambda^s<1$. With a small step size $\eta^s = {\cal O}\big(1/(\kappa L)\big)$ though, choosing $m^s = {\cal O}\big(1/(\mu \eta^s) \big)$ also leads to $\lambda^s <1$. These considerations prompt us to adopt a time-varying inner loop length adjusted by $\eta^s$ in \eqref{eq.bb_stepsize} as
\begin{align}\label{eq.dyna_bb}
	 m^s = \frac{c}{\mu \eta^s}\;.
\end{align}
Such choices of $\eta^s$ and $m^s$ at first glance do not lead to a tune-free algorithm directly, because one has to find an optimal $\theta_\kappa$ and $c$ through tuning. Fortunately, there are simple choices for both $c$ and $\theta_\kappa$. In Propositions \ref{prop.bbsvrg} and \ref{prop.bbsarah}, the smallest selected $\theta_\kappa$ for SVRG and SARAH with different types of averaging turns out to be a reliable choice; while choosing $c=1$ has been good enough throughout our numerical experiments. Although the selection of these parameters violates slightly the theoretical guarantee, its merits lie in the simplicity. And in our experiments, no divergence has been observed by these parameter selections. 

\textbf{Averaging schemes.} As discussed in subsection \ref{sec.avg}, W-Avg gains in performance when either $m^s$ or $\eta^s$ is large. Since $m^s$ and $\eta^s$ are inversely proportional (cf. \eqref{eq.dyna_bb}), it is clear that one of the two suffices to be large; and for this reason, we will rely on W-Avg for BB-SARAH.

Extensions regarding almost tune-free variance reduction for (non)convex problems can be found in our technical note \citep{li2019adaptive}.

\section{Numerical Tests}\label{sec.tests}
\begin{figure*}[t]
	\centering
	\begin{tabular}{ccc}
		\hspace{-0.2cm}
		\includegraphics[width=.32\textwidth]{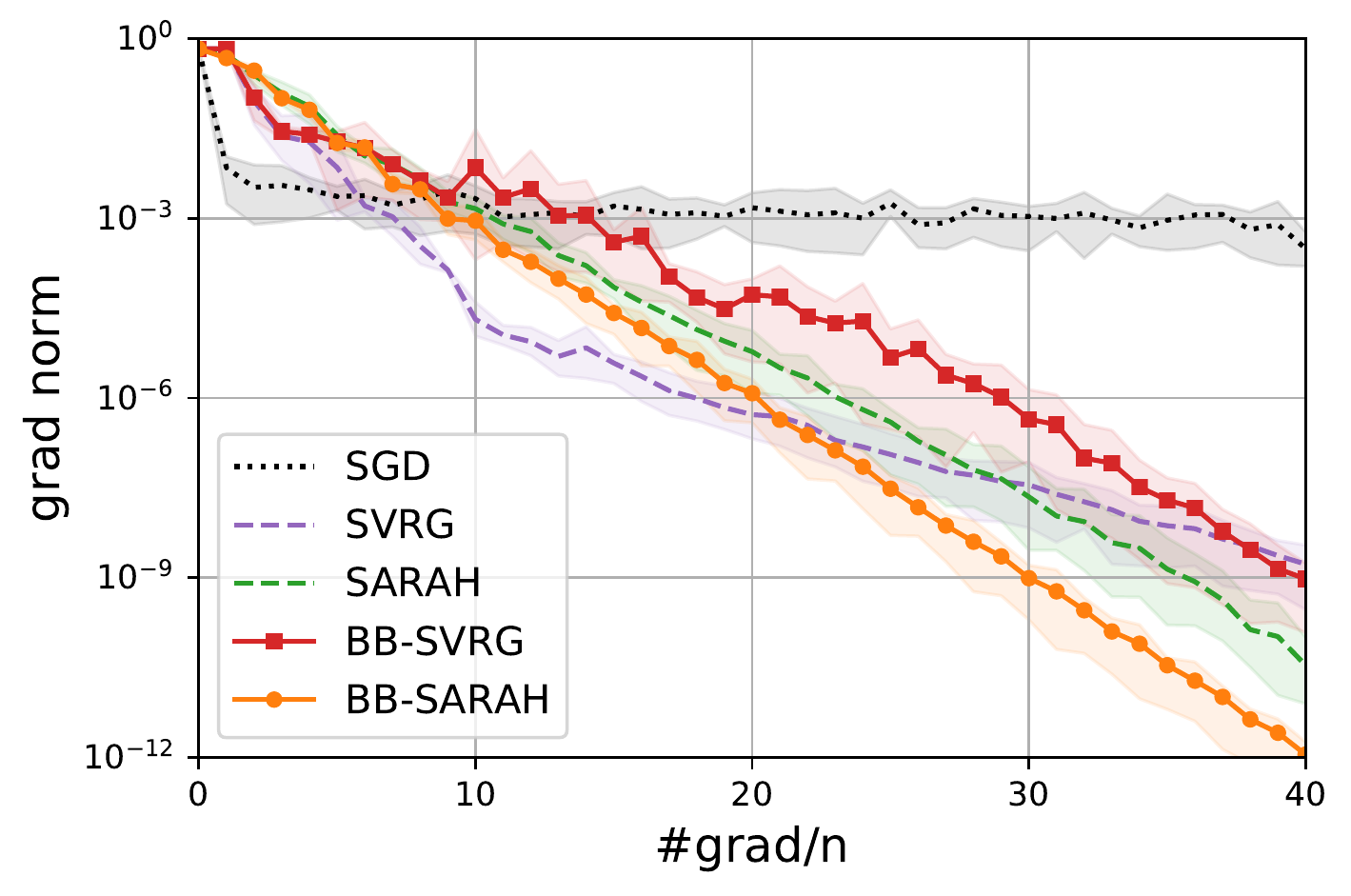}&
		\hspace{-0.2cm}
		\includegraphics[width=.32\textwidth]{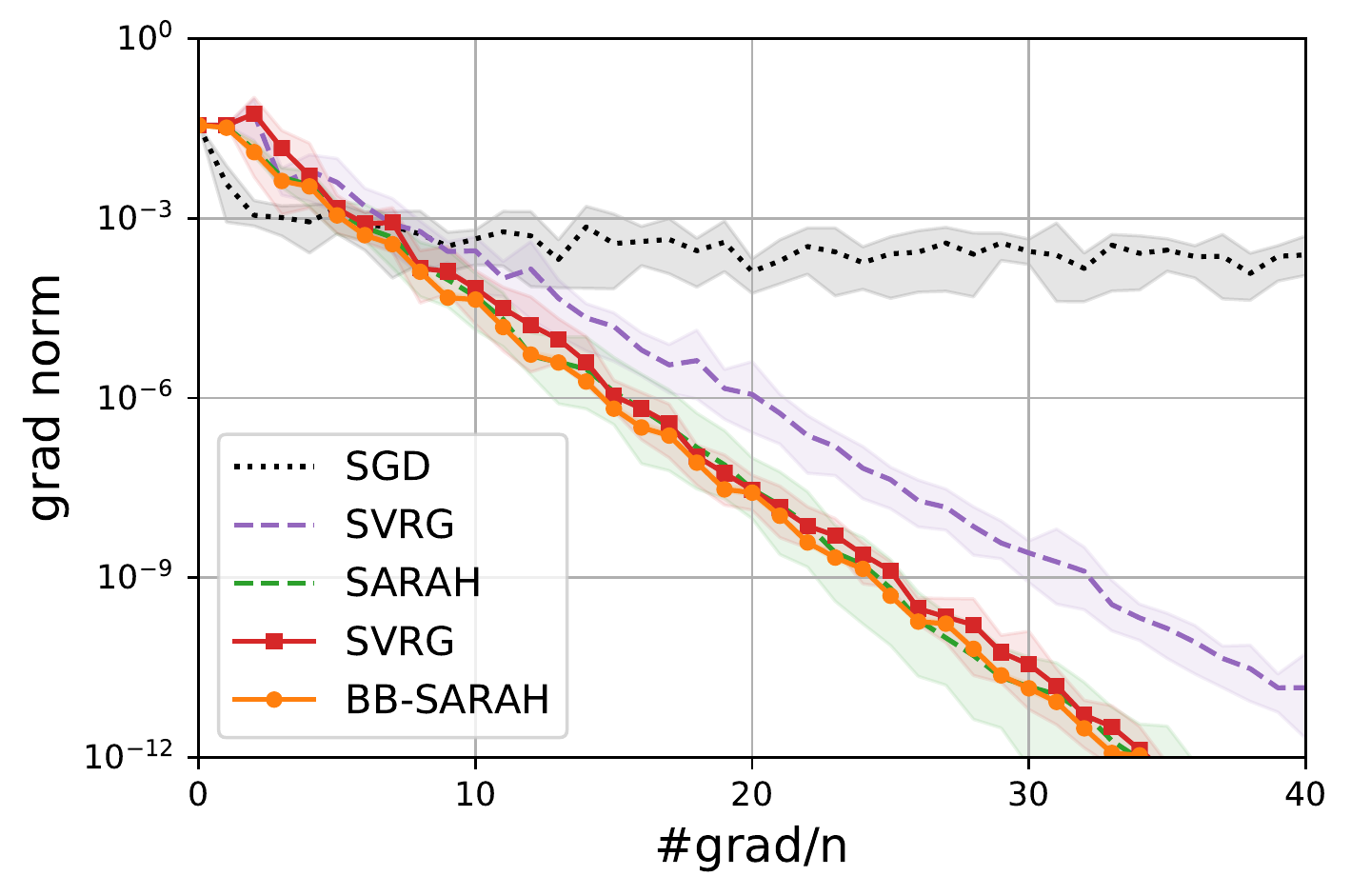}&
		\hspace{-0.3cm}
		\includegraphics[width=.32\textwidth]{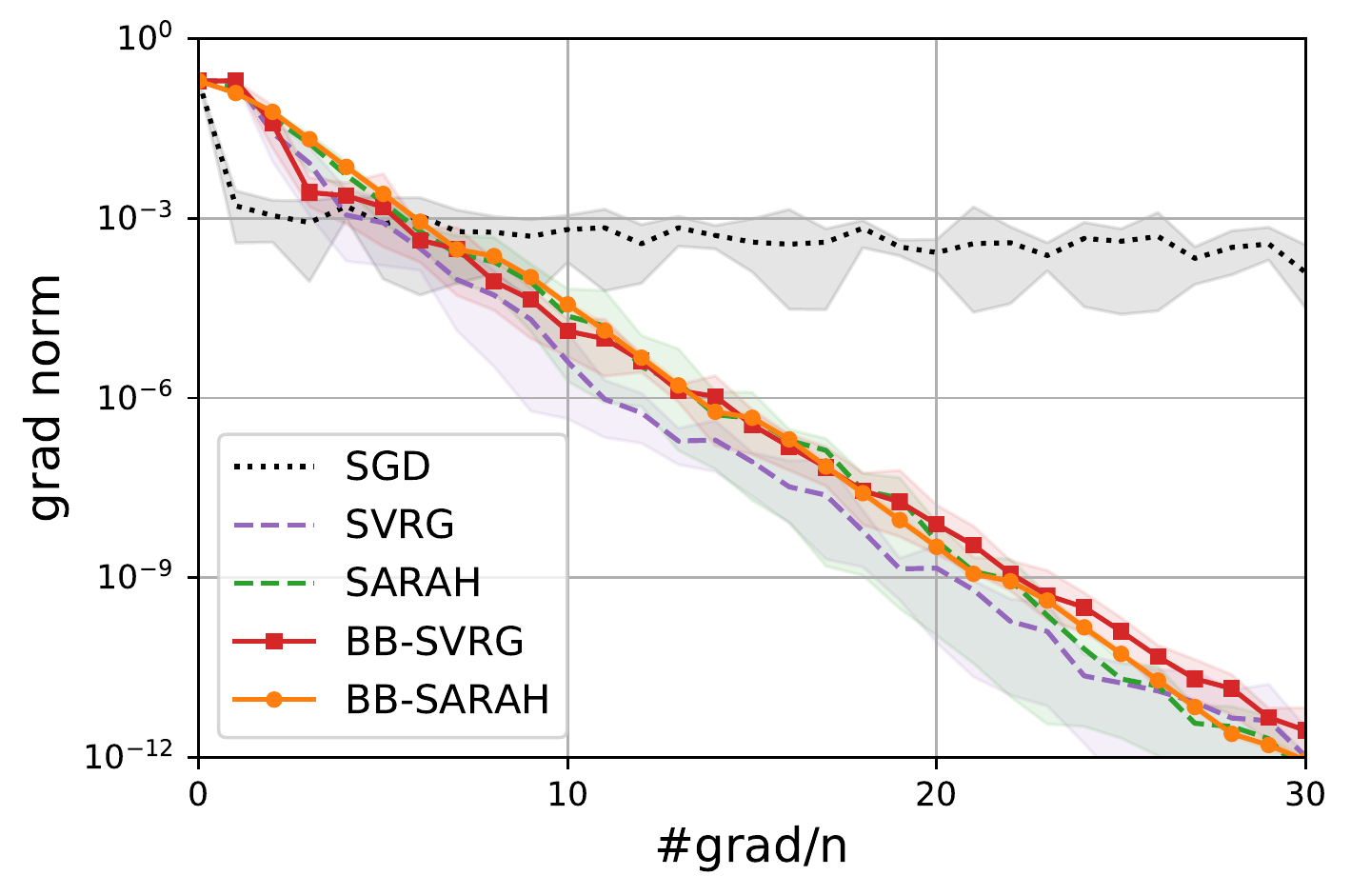}
		\\ (a) \textit{a9a}  & (b) \textit{rcv1} & (c) \textit{real-sim}
	\end{tabular}
	\caption{Tests of BB-SVRG and BB-SARAH on different datasets.}
	 \label{fig.bb}
\end{figure*}

To assess performance, the proposed tune-free BB-SVRG and BB-SARAH are applied to binary classification via regularized logistic regression (cf. \eqref{eq.test}) using the datasets \textit{a9a}, \textit{rcv1.binary}, and \textit{real-sim} from LIBSVM\footnote{Online available at \url{https://www.csie.ntu.edu.tw/~cjlin/libsvmtools/datasets/binary.html}.}. Details regarding the datasets, the $\mu$ values used, and implementation details are deferred to Appendix \ref{appdx.final_tests}. 

For comparison, the selected benchmarks are SGD, SVRG with U-Avg, and SARAH with U-Avg. The step size for SGD is $\eta = 0.05/(L(n_{e}+1))$, where $n_e$ is the index of epochs. For SVRG and SARAH, we fix $m=5\kappa$, and tune for the best step sizes. For BB-SVRG, we choose $\eta^s$ and $m^s$ as \eqref{eq.dyna_bb} with $\theta_\kappa = 4\kappa$ (as in Proposition \ref{prop.bbsvrg}) and $c = 1$. While we choose $\theta_\kappa = \kappa$ (as in Proposition \ref{prop.bbsarah}) and $c = 1$ for BB-SARAH. W-Avg is adopted for both BB-SVRG and BB-SARAH.

The results are showcased in Fig. \ref{fig.bb}. We also tested BB-SVRG with parameters chosen as \citep[Thm. 3.8]{tan2016}. However it only slightly outperforms SGD and hence is not plotted here (see the blue line in Fig. \ref{fig.m_choices}(a) as a reference). On dataset \textit{a9a}, BB-SARAH outperforms tuned SARAH. BB-SVRG is worse than SVRG initially, but has similar performance around the $40$th sample pass on the x-axis. On dataset \textit{rcv1} however, BB-SARAH, BB-SVRG and SARAH have similar performance, improving over SVRG. On dataset \textit{real-sim}, BB-SARAH performs almost identical to SARAH. BB-SVRG exhibits comparable performance with SVRG.

\section{Conclusions}
Almost tune-free SVRG and SARAH were developed in this work. Besides the BB step size for eliminating the tuning for step size, the key insights are that both i) averaging, as well as ii) the number of inner loop iterations should be adjusted according to the BB step size. Specific major findings include: i) estimate sequence based provably linear convergence of SVRG and SARAH, which enabled new types of averaging for efficient variance reduction; ii) theoretical guarantees of BB-SVRG and BB-SARAH with different types of averaging; and, iii) implementable tune-free variance reduction algorithms. The efficacy of the tune-free BB-SVRG and BB-SARAH were corroborated numerically.
\bibliographystyle{plainnat}
\bibliography{myabrv,datactr}

\newpage
\onecolumn
\appendix
\begin{center}
{\Huge  \textbf{Appendix} }
\end{center}

\section{Properties of ES}
\subsection{Proof of Lemma \ref{lemma.est_seq}}
	i) By definition $\Phi_0(\mathbf{x})$ is $\mu_0$-strongly convex; and by checking Hession one can find that $\Phi_k(\mathbf{x})$ is 
	$\mu_k$-strongly convex with $\mu_k = (1-\delta_k) \mu_{k-1} + \delta_k \mu$.
	
	ii) Clearly, $\mathbf{x}_0$ minimizes $\Phi_0(\mathbf{x})$, and $\Phi_k(\mathbf{x})$ is quadratic. Arguing by induction, suppose that $\mathbf{x}_{k-1}$ minimizes $\Phi_{k-1}(\mathbf{x})$, to obtain
	\begin{align*}
		\Phi_{k-1} (\mathbf{x}) = \Phi_{k-1}^* + \frac{\mu_{k-1}}{2} \| \mathbf{x} - \mathbf{x}_{k-1} \|^2 ~~~\Rightarrow~~~ \nabla \Phi_{k-1} (\mathbf{x}) = \mu_{k-1} ( \mathbf{x} -  \mathbf{x}_{k-1} ).
	\end{align*}
By definition of $ \Phi_k(\mathbf{x})$, we also have
	\begin{align}\label{eq.gphi_k}
		\nabla \Phi_k(\mathbf{x}) & = (1- \delta_k) \nabla \Phi_{k-1} (\mathbf{x}) + \delta_k  \mathbf{v}_{k-1} + \mu  \delta_k ( \mathbf{x} -  \mathbf{x}_{k-1} ) \nonumber \\
		& = (1- \delta_k) \mu_{k-1} ( \mathbf{x} -  \mathbf{x}_{k-1} ) + \delta_k  \mathbf{v}_{k-1} + \mu  \delta_k ( \mathbf{x} -  \mathbf{x}_{k-1} ).
	\end{align}
	Using $\mu_k = (1-\delta_k) \mu_{k-1} + \delta_k \mu$ and setting $\nabla \Phi_k(\mathbf{x})= {\bf 0}$, we find that $\mathbf{x}_k$ minimizes $\Phi_k(\mathbf{x})$ when $\delta_k = \eta \mu_k$.
	
	iii) Since $\mathbf{x}_{k-1}$ minimizes $\Phi_{k-1}(\mathbf{x})$, using the definition of $ \Phi_k(\mathbf{x})$ we can write
	\begin{align}\label{eq.phikk}
		\Phi_k (\mathbf{x}_{k-1}) =  (1- \delta_k)\Phi_{k-1}^* + \delta_k  f(\mathbf{x}_{k-1}).
	\end{align}
On the other hand, we also have $\Phi_k (\mathbf{x}_{k-1}) = \Phi_k^* + \frac{\mu_k}{2} \| \mathbf{x}_{k-1} - \mathbf{x}_k \|^2 $. Comparing this with \eqref{eq.phikk} and using that $\mathbf{x}_k = \mathbf{x}_{k-1} -\eta \mathbf{v}_{k-1}$, completes the proof of this property.

\subsection{Derivations of \eqref{eq.svrgeg} and \eqref{eq.saraheg}}\label{apdx.eg}
To verify \eqref{eq.svrgeg}, proceed as follows 
\begin{align}
	& ~~~~\mathbb{E}\big[ \Phi_k (\mathbf{x}) \big] = (1- \delta)\mathbb{E}\big[ \Phi_{k-1} (\mathbf{x}) \big]
	+ \delta\mathbb{E} \bigg[  f(\mathbf{x}_{k-1}) + \langle  \mathbf{v}_{k-1}, \mathbf{x} - \mathbf{x}_{k-1}  \rangle + \frac{\mu}{2} \| \mathbf{x} -  \mathbf{x}_{k-1} \|^2	\bigg] \nonumber \\
	&= (1- \delta)\mathbb{E}\big[ \Phi_{k-1} (\mathbf{x}) \big]
	+ \delta\mathbb{E} \bigg[  f(\mathbf{x}_{k-1}) + \langle  \nabla f(\mathbf{x}_{k-1}), \mathbf{x} - \mathbf{x}_{k-1}  \rangle + \frac{\mu}{2} \| \mathbf{x} -  \mathbf{x}_{k-1} \|^2	\bigg] \nonumber \\
	& \leq (1- \delta)\mathbb{E}\big[ \Phi_{k-1} (\mathbf{x}) \big]
	+ \delta f(\mathbf{x}) \leq (1-\delta)^k \big[\Phi_0(\mathbf{x}) - f(\mathbf{x})\big] + f(\mathbf{x}) \leq (1-\delta)^k \big[\Phi_0(\mathbf{x}) - f(\mathbf{x}^*)\big] + f(\mathbf{x}).
\end{align}

And in order to derive \eqref{eq.saraheg}, follow the next steps 
\begin{align*}
	& ~~~~\mathbb{E}\big[ \Phi_k (\mathbf{x}) \big] = (1- \delta)\mathbb{E}\big[ \Phi_{k-1} (\mathbf{x}) \big]
	+ \delta\mathbb{E} \bigg[  f(\mathbf{x}_{k-1}) + \langle  \mathbf{v}_{k-1}, \mathbf{x} - \mathbf{x}_{k-1}  \rangle + \frac{\mu}{2} \| \mathbf{x} -  \mathbf{x}_{k-1} \|^2	\bigg] \nonumber \\
	& \leq (1- \delta)\mathbb{E}\big[ \Phi_{k-1} (\mathbf{x}) \big]
	+ \delta f(\mathbf{x}) + \delta \mathbb{E}\big[ \langle \mathbf{v}_{k-1} - \nabla f(\mathbf{x}_{k-1}), \mathbf{x} - \mathbf{x}_{k-1}\rangle \big] \nonumber \\
	& \leq (1-\delta)^k \big[\Phi_0(\mathbf{x}) - f(\mathbf{x})\big] + f(\mathbf{x}) + \underbrace{\delta \sum_{\tau =0}^{k-1} (1 - \delta)^\tau \mathbb{E}\big[ \langle \mathbf{v}_{k-1-\tau} - \nabla f(\mathbf{x}_{k-1-\tau}), \mathbf{x} - \mathbf{x}_{k-1-\tau}\rangle \big] }_{:=C;~~ C \neq 0, ~\text{an extra term compared with SVRG}}.
\end{align*}

\subsection{A Key Lemma}
The next lemma plays a major role in our analysis. 
\begin{lemma}\label{lemma.est_seq2}
If we choose $\mu_0 = \mu$, $\delta_k = \mu_k \eta$, and $\Phi_0^* = f(\mathbf{x}_0)$ in the ES defined in \eqref{eq.est_seq}, we then find that: i) $\mu_k = \mu, \forall k$; ii) $\delta:= \delta_k = \mu\eta$; and iii) the following inequality holds
	\begin{align*}
		& ~~ \delta \sum_{\tau = 1}^{k-1}  (1 - \delta)^{k-\tau-1} \big[   f(\mathbf{x}_\tau) -  f(\mathbf{x}^*)  \big] + (1 - \delta)^{k-1} \big[   f(\mathbf{x}_0) -  f(\mathbf{x}^*)  \big] \nonumber \\
		& \leq (1 - \delta)^k \big[ \Phi_0(\mathbf{x}^*) - f(\mathbf{x}^*) \big] +  \frac{ \mu \eta^2 }{2}  \sum_{\tau = 1}^k  (1 - \delta)^{k-\tau} \| \mathbf{v}_{\tau-1} \|^2  + \delta \sum_{\tau=1}^k (1 - \delta)^{k-\tau} \zeta_{\tau - 1}
	\end{align*}
	where $\zeta_{k-1}:= \langle  \mathbf{v}_{k-1} - \nabla f(\mathbf{x}_{k-1}), \mathbf{x}^* - \mathbf{x}_{k-1}  \rangle$.
\end{lemma}
\begin{proof}
Since i) and ii) are straightforward to verify, we will prove iii). Using property iii) in Lemma \ref{lemma.est_seq}, we find
	\begin{align}\label{eq.svrg_pf_step1}
		f(\mathbf{x}_k) - \Phi_k^* & = f(\mathbf{x}_k) - (1 - \delta_k )\Phi_{k-1}^* - \delta_k f(\mathbf{x}_{k-1}) + \frac{ \mu_k \eta^2 }{2} \| \mathbf{v}_{k-1} \|^2	\nonumber \\
		& = f(\mathbf{x}_k) - \Phi_{k-1}^* +  \delta_k \big(\Phi_{k-1}^* -f(\mathbf{x}_{k-1})\big) + \frac{ \mu_k \eta^2 }{2} \| \mathbf{v}_{k-1} \|^2	\nonumber \\
		& = f(\mathbf{x}_k) -  f(\mathbf{x}_{k-1}) +  f(\mathbf{x}_{k-1}) - \Phi_{k-1}^* +  \delta_k \big(\Phi_{k-1}^* -f(\mathbf{x}_{k-1})\big) + \frac{ \mu_k \eta^2 }{2} \| \mathbf{v}_{k-1} \|^2	\nonumber \\
		& = (1 - \delta_k) \big[ f(\mathbf{x}_{k-1}) - \Phi_{k-1}^*  \big] + \xi_k
	\end{align}
where $\xi_k$ is defined as
	\begin{align*}
		\xi_k:= f(\mathbf{x}_k) -  f(\mathbf{x}_{k-1}) +  \frac{ \mu_k \eta^2 }{2} \| \mathbf{v}_{k-1} \|^2. 
	\end{align*}
Upon expanding $ f(\mathbf{x}_{k-1}) - \Phi_{k-1}^* $ in \eqref{eq.svrg_pf_step1}, we have
	\begin{align}\label{eq.svrg_pf_step2}
		f(\mathbf{x}_k) - \Phi_k^* & = (1 - \delta_k) \big[ f(\mathbf{x}_{k-1}) - \Phi_{k-1}^*  \big] + \xi_k \nonumber \\
		& = \Big[ \prod_{\tau = 1}^k (1 - \delta_\tau) \Big] [f(\mathbf{x}_0) - \Phi_0^*] + \sum_{\tau = 1}^k \xi_\tau \Big[\prod_{j= \tau+1}^k (1 - \delta_j) \Big]
	\end{align}
from which we deduce that  
	\begin{align}\label{eq.phik}
		\Phi_k^* & \leq \Phi_k(\mathbf{x}^*) = (1- \delta_k)\Phi_{k-1} (\mathbf{x}^*) + \delta_k \bigg[  f(\mathbf{x}_{k-1}) + \langle  \mathbf{v}_{k-1}, \mathbf{x}^* - \mathbf{x}_{k-1}  \rangle + \frac{\mu}{2} \| \mathbf{x}^* -  \mathbf{x}_{k-1} \|^2	\bigg] \nonumber \\
		&  \stackrel{(a)}{=} (1- \delta_k)\Phi_{k-1} (\mathbf{x}^*) + \delta_k \bigg[  f(\mathbf{x}_{k-1}) + \langle  \nabla f(\mathbf{x}_{k-1}), \mathbf{x}^* - \mathbf{x}_{k-1}  \rangle + \frac{\mu}{2} \| \mathbf{x}^* -  \mathbf{x}_{k-1} \|^2	+ \zeta_{k-1} \bigg] \nonumber \\
		&  \stackrel{(b)}{\leq} (1- \delta_k)\Phi_{k-1} (\mathbf{x}^*) + \delta_k  f(\mathbf{x}^*) 	+  \delta_k \zeta_{k-1} \nonumber \\
		& \leq \Big[ \prod_{\tau = 1}^k (1 - \delta_\tau) \Big] \Phi_0(\mathbf{x}^*) + \sum_{\tau = 1}^k \delta_\tau f(\mathbf{x}^*) \Big[\prod_{j= \tau+1}^k (1 - \delta_j) \Big]+ \sum_{\tau = 1}^k \delta_\tau \zeta_{\tau-1} \Big[\prod_{j= \tau+1}^k (1 - \delta_j) \Big]
	\end{align}
	where in (a) the $\zeta_{k-1}$ is defined as 
	\begin{align*}
		\zeta_{k-1}:= \langle  \mathbf{v}_{k-1} - \nabla f(\mathbf{x}_{k-1}), \mathbf{x}^* - \mathbf{x}_{k-1}  \rangle;
	\end{align*}
and (b) follows from the strongly convexity of $f$. Then, using \eqref{eq.svrg_pf_step2}, we have
	\begin{align*}
		f(\mathbf{x}_k) & - f(\mathbf{x}^*)  =   \Phi_k^* - f(\mathbf{x}^*) + \Big[ \prod_{\tau = 1}^k (1 - \delta_\tau) \Big] [f(\mathbf{x}_0) - \Phi_0^*] + \sum_{\tau = 1}^k \xi_\tau \Big[\prod_{j= \tau+1}^k (1 - \delta_j) \Big]		\nonumber \\
		& \stackrel{(c)}{\leq} \Big[ \prod_{\tau = 1}^k (1 - \delta_\tau) \Big] \Phi_0(\mathbf{x}^*) + \sum_{\tau = 1}^k \delta_\tau f(\mathbf{x}^*) \Big[\prod_{j= \tau+1}^k (1 - \delta_j) \Big]+ \sum_{\tau = 1}^k \delta_\tau \zeta_{\tau-1} \Big[\prod_{j= \tau+1}^k (1 - \delta_j) \Big] \nonumber \\
		&~~~~~~~~~~~~ - f(\mathbf{x}^*) + \Big[ \prod_{\tau = 1}^k (1 - \delta_\tau) \Big] [f(\mathbf{x}_0) - \Phi_0^*] + \sum_{\tau = 1}^k \xi_\tau \Big[\prod_{j= \tau+1}^k (1 - \delta_j) \Big]	
	\end{align*}
where (c) is due to \eqref{eq.phik}. Choosing $\mu_0 = \mu$ (hence $\mu_k = \mu, \delta_k = \mu \eta:= \delta,~ \forall k$) and $\Phi_0^* = f(\mathbf{x}_0)$, we arrive at
	\begin{align}\label{eq.goal1}
		f(\mathbf{x}_k) & - f(\mathbf{x}^*)  \leq  (1 - \delta)^k \big[ \Phi_0(\mathbf{x}^*) - f(\mathbf{x}^*) \big] + \sum_{\tau = 1}^k  (1 - \delta)^{k-\tau} \big( \xi_\tau  + \delta \zeta_{\tau-1} \big).
	\end{align}
	Now consider that
	\begin{align}\label{eq.xi_t}
		& ~~~~ \sum_{\tau = 1}^k  (1 - \delta)^{k-\tau}  \xi_\tau  = \sum_{\tau = 1}^k  (1 - \delta)^{k-\tau} \Big[   f(\mathbf{x}_\tau) -  f(\mathbf{x}_{\tau-1}) +  \frac{ \mu \eta^2 }{2} \| \mathbf{v}_{\tau-1} \|^2  \Big]  	 \nonumber \\
		& =  f(\mathbf{x}_k) + \sum_{\tau = 1}^{k-1}  (1 - \delta)^{k-\tau}    f(\mathbf{x}_\tau) - \sum_{\tau = 1}^{k-1}  (1 - \delta)^{k-\tau-1}   f(\mathbf{x}_\tau)  -(1 - \delta)^{k-1} f(\mathbf{x}_0) + \frac{ \mu \eta^2 }{2}  \sum_{\tau = 1}^k  (1 - \delta)^{k-\tau} \| \mathbf{v}_{\tau-1} \|^2  \nonumber \\
		& = - \delta \sum_{\tau = 1}^{k-1}  (1 - \delta)^{k-\tau-1}   f(\mathbf{x}_\tau) + f(\mathbf{x}_k) - (1 - \delta)^{k-1} f(\mathbf{x}_0) +  \frac{ \mu \eta^2 }{2}  \sum_{\tau = 1}^k  (1 - \delta)^{k-\tau} \| \mathbf{v}_{\tau-1} \|^2.
	\end{align}
Because $\delta \sum_{\tau = 1}^{k-1}  (1 - \delta)^{k-\tau-1} + (1 - \delta)^{k-1} = 1$, we can write $f(\mathbf{x}^*) = [ \delta \sum_{\tau = 1}^{k-1}  (1 - \delta)^{k-\tau-1} + (1 - \delta)^{k-1} ] f(\mathbf{x}^*)$. Using the latter, plugging \eqref{eq.xi_t} into \eqref{eq.goal1}, and eliminating $f(\mathbf{x}_k)$, we obtain 
	\begin{align}\label{eq.ck}
		& ~~ \delta \sum_{\tau = 1}^{k-1}  (1 - \delta)^{k-\tau-1} \big[   f(\mathbf{x}_\tau) -  f(\mathbf{x}^*)  \big] + (1 - \delta)^{k-1} \big[   f(\mathbf{x}_0) -  f(\mathbf{x}^*)  \big] \nonumber \\
		& \leq (1 - \delta)^k \big[ \Phi_0(\mathbf{x}^*) - f(\mathbf{x}^*) \big] +  \frac{ \mu \eta^2 }{2}  \sum_{\tau = 1}^k  (1 - \delta)^{k-\tau} \| \mathbf{v}_{\tau-1} \|^2  + \delta \sum_{\tau=1}^k (1 - \delta)^{k-\tau} \zeta_{\tau - 1}
	\end{align}
which completes the proof.
\end{proof}

\section{Proofs for SVRG and SARAH}
\subsection{Proof for SVRG (Theorem \ref{thm.svrg} and Corollary \ref{coro.svrg})}

\textbf{Proof of Theorem \ref{thm.svrg}}
\begin{proof}
Since the choices of $\mu_0$, $\Phi_0^*$, and $\delta_k$ coincide with those in Lemma \ref{lemma.est_seq2}, we can directly apply Lemma \ref{lemma.est_seq2} to find
\begin{align}\label{eq.svrg_base}
		& ~~ \delta \sum_{\tau = 1}^{k-1}  (1 - \delta)^{k-\tau-1} \big[   f(\mathbf{x}_\tau) -  f(\mathbf{x}^*)  \big] + (1 - \delta)^{k-1} \big[   f(\mathbf{x}_0) -  f(\mathbf{x}^*)  \big] \nonumber \\
		& \leq (1 - \delta)^k \big[ \Phi_0(\mathbf{x}^*) - f(\mathbf{x}^*) \big] +  \frac{ \mu \eta^2 }{2}  \sum_{\tau = 1}^k  (1 - \delta)^{k-\tau} \| \mathbf{v}_{\tau-1} \|^2  + \delta \sum_{\tau=1}^k (1 - \delta)^{k-\tau} \zeta_{\tau - 1}
\end{align}
where $\zeta_{k-1}:= \langle  \mathbf{v}_{k-1} - \nabla f(\mathbf{x}_{k-1}), \mathbf{x}^* - \mathbf{x}_{k-1}  \rangle $. Upon defining the $\sigma$-algebra ${\cal F}_{k-1} = \sigma(i_0, i_1, \ldots,i_{k-1})$, and using that $\mathbf{v}_k$ is an unbiased estimate of $\nabla f(\mathbf{x}_k)$, it follows readily that 
\begin{align*}
	\mathbb{E}[\zeta_k| {\cal F}_{k-1}] = \mathbb{E}\big[\mathbf{v}_k - \nabla f(\mathbf{x}_k), \mathbf{x}^* - \mathbf{x}_k  \rangle| {\cal F}_{k-1}\big] = 0
\end{align*}
which further implies
\begin{align}\label{eq.svrg_cancel}
	\mathbb{E}[\zeta_k] = 0.
\end{align}
Now taking expectation on both sides of \eqref{eq.svrg_base} and using \eqref{eq.svrg_cancel}, we have
\begin{align}\label{eq.23}
	 & \delta \sum_{\tau = 1}^{k-1}  (1 - \delta)^{k-\tau-1} \mathbb{E} \big[   f(\mathbf{x}_\tau) -  f(\mathbf{x}^*)  \big] + (1 - \delta)^{k-1} \mathbb{E} \big[   f(\mathbf{x}_0) -  f(\mathbf{x}^*)  \big]   \\
	& \leq (1 - \delta)^k \mathbb{E}\big[ \Phi_0(\mathbf{x}^*) - f(\mathbf{x}^*) \big] +  \frac{ \mu \eta^2 }{2}  \sum_{\tau = 1}^k  (1 - \delta)^{k-\tau} \mathbb{E} \big[ \| \mathbf{v}_{\tau-1} \|^2 \big] \nonumber \\
	&  \stackrel{(a)}{\leq} (1 - \delta)^k \mathbb{E}\big[ \Phi_0(\mathbf{x}^*) - f(\mathbf{x}^*) \big] +  2 \mu L \eta^2   \sum_{\tau = 0}^{k-1}  (1 - \delta)^{k-\tau-1} \mathbb{E} \big[ f(\mathbf{x}_{\tau}) - f(\mathbf{x}^*) + f(\mathbf{x}_{0}) - f(\mathbf{x}^*) \big]    \nonumber \\
	& \stackrel{(b)}{\leq}(1 - \delta)^k \mathbb{E}\big[ \Phi_0(\mathbf{x}^*) - f(\mathbf{x}^*) \big] +  2 \mu L \eta^2   \sum_{\tau = 0}^{k-1}  (1 - \delta)^{k-\tau-1} \mathbb{E} \big[ f(\mathbf{x}_{\tau}) - f(\mathbf{x}^*) \big] +  \frac{2 \mu L \eta^2 }{\delta} \mathbb{E}\big[ f(\mathbf{x}_{0}) - f(\mathbf{x}^*)\big] \nonumber
\end{align}
where in (a) we used Lemma \ref{lemma.est_err} to $\mathbb{E} [ \| \mathbf{v}_{\tau-1} \|^2 ]$; and (b) holds because $\sum_{\tau = 0}^{k-1}  (1 - \delta)^{k-\tau-1} \leq 1/\delta$. Note that we can use $\Phi_0(\mathbf{x}^*) = f(\mathbf{x}_0)+\frac{\mu}{2}\| \mathbf{x}_0 - \mathbf{x}^* \|^2$ together with $(1-\delta)^{k-1} > (1-\delta)^k$, to eliminate $(1 - \delta)^{k-1} \mathbb{E}[  f(\mathbf{x}_0) -  f(\mathbf{x}^*)  ] $ on the LHS of \eqref{eq.23}. Rearranging the terms, we arrive at
\begin{align}
	& ~~~ (\delta - 2 \mu L \eta^2) \sum_{\tau = 1}^{k-1}  (1 - \delta)^{k-\tau-1} \mathbb{E} \big[   f(\mathbf{x}_\tau) -  f(\mathbf{x}^*)  \big]   \nonumber \\
	& \leq \frac{\mu}{2}(1 - \delta)^k \mathbb{E}\big[ \| \mathbf{x}_0 - \mathbf{x}^* \|^2 \big] +  2 \mu L \eta^2   (1 - \delta)^{k-1} \mathbb{E} \big[ f(\mathbf{x}_0) - f(\mathbf{x}^*) \big] +  \frac{2 \mu L \eta^2 }{\delta} \mathbb{E}\big[ f(\mathbf{x}_{0}) - f(\mathbf{x}^*)\big] \nonumber \\
	& \leq \bigg[ (1 - \delta)^k + 2 \mu L \eta^2   (1 - \delta)^{k-1} +  \frac{2 \mu L \eta^2 }{\delta} \bigg]  \mathbb{E}\big[ f(\mathbf{x}_{0}) - f(\mathbf{x}^*)\big] 
\end{align}
where the last inequality is due to $\frac{\mu}{2}\| \mathbf{x} - \mathbf{x}^* \|\leq f(\mathbf{x}) - f(\mathbf{x}^*)$. Now choosing $\eta< 1/ 2L$ so that $\delta - 2 \mu L\eta^2 >0$, we have 
\begin{align*}
	\sum_{\tau = 1}^{k-1}  (1 - \delta)^{k-\tau-1} \mathbb{E} \big[   f(\mathbf{x}_\tau) -  f(\mathbf{x}^*)  \big]   \leq \bigg[ \frac{(1 - \delta)^k}{ \delta - 2 \mu L \eta^2 } + \frac{2 \mu L \eta^2   (1 - \delta)^{k-1}}{\delta - 2 \mu L \eta^2 } +  \frac{2 \mu L \eta^2 }{\delta(\delta - 2 \mu L \eta^2) } \bigg]  \mathbb{E}\big[ f(\mathbf{x}_{0}) - f(\mathbf{x}^*)\big].
\end{align*}
With $p_0 = p_m = 0$, and $p_k = (1 - \delta)^{m-k -1}/q, k = 1,2,\ldots, m-1$, where $q = [1 - (1 - \delta)^{m-1}]/\delta$ (with $\delta = \mu\eta$), we find  
\begin{align}
	& \mathbb{E} \big[ f(\tilde{\mathbf{x}}^{s}) - f(\mathbf{x}^*) \big] = \sum_{\tau = 1}^{m-1}  \frac{(1 - \delta)^{m-\tau-1}}{q} \mathbb{E} \big[   f(\mathbf{x}_\tau) -  f(\mathbf{x}^*)  \big] \nonumber \\
	& \leq \frac{1}{q}\bigg[ \frac{(1 - \delta)^m}{ \delta - 2 \mu L \eta^2 } + \frac{2 \mu L \eta^2   (1 - \delta)^{m-1}}{\delta - 2 \mu L \eta^2 } +  \frac{2 \mu L \eta^2 }{\delta(\delta - 2 \mu L \eta^2) } \bigg]  \mathbb{E}\big[ f(\tilde{\mathbf{x}}^{s-1}) - f(\mathbf{x}^*)\big] \nonumber \\
	& = \underbrace{\frac{1}{1 - (1 - \mu \eta)^{m-1}} \bigg[ \frac{(1 - \mu \eta)^m}{ 1 - 2 \eta L } + \frac{2 \mu L \eta^2   (1 - \mu \eta )^{m-1}}{1 - 2  L \eta } +  \frac{2  L \eta }{1 - 2 L \eta } \bigg]}_{:= \lambda^{\texttt{SVRG}}}  \mathbb{E}\big[ f(\tilde{\mathbf{x}}^{s-1}) - f(\mathbf{x}^*)\big].
\end{align}
Thus, so long as we choose a large enough $m$ and $\eta < 1/(4L)$, we have $\lambda^{\texttt{SVRG}}<1$, that is, SVRG converges linearly.
\end{proof}

\noindent\textbf{Proof of Corollary \ref{coro.svrg}}
\begin{proof}
Choose $\eta = 1/(8L)$ and $m = \frac{3}{\mu \eta} +1 = 24 \kappa +1 \geq 25$. We have that 
\begin{align*}
	(1 - \mu \eta)^{\frac{1}{\mu\eta}}	\leq 0.4 ~~~\Rightarrow~~~ (1 - \mu \eta)^{m} \leq (0.4)^3 
\end{align*}
(Actually $(1 - \mu \eta)^{\frac{1}{\mu\eta}} \approx 1/e$ when $\mu \eta$ is small enough). Using the value of $\eta$ and $m$, it can be verified that $\lambda^{\texttt{SVRG}} \leq 0.5$. This implies that ${\cal O}\big(\ln\frac{1}{\epsilon} \big)$ outer loops are needed for an $\epsilon$-accurate solution. And since $m = {\cal O}(\kappa)$, the overall complexity is ${\cal O}\big( (n+\kappa)\ln\frac{1}{\epsilon} \big)$.
\end{proof}

\subsection{Proofs for SARAH (Lemma \ref{lemma.sarah_addition}, Theorem \ref{thm.sarah} and Corollary \ref{coro.sarah})}\label{appdx.pf_sarah}
\textbf{Proof of Lemma \ref{lemma.sarah_addition}}
\begin{proof}
	Let ${\cal F}_{k-1} = \sigma(i_1, i_2,\ldots,i_{k-1})$, then for any $\mathbf{x}$ we have
	\begin{align*}
		\mathbb{E} \big[ \langle \mathbf{v}_k - \nabla f(\mathbf{x}_k),~ & \mathbf{x} - \mathbf{x}_k \rangle  | {\cal F}_{k-1}\big]	 = \mathbb{E} \big[ \langle \nabla f_{i_k}(\mathbf{x}_k) - \nabla f_{i_k}(\mathbf{x}_{k-1}) + \mathbf{v}_{k-1} - \nabla f(\mathbf{x}_k), ~\mathbf{x} - \mathbf{x}_k \rangle  | {\cal F}_{k-1}\big]  \\
		& = \langle  \mathbf{v}_{k-1} - \nabla f(\mathbf{x}_{k-1}) ,~ \mathbf{x} - \mathbf{x}_k \rangle = \langle  \mathbf{v}_{k-1} - \nabla f(\mathbf{x}_{k-1}) ,~ \mathbf{x} - \mathbf{x}_{k-1} + \mathbf{x}_{k-1} - \mathbf{x}_k \rangle \nonumber \\
		& = \langle  \mathbf{v}_{k-1} - \nabla f(\mathbf{x}_{k-1}) ,~ \mathbf{x} - \mathbf{x}_{k-1} \rangle + \eta  \langle  \mathbf{v}_{k-1} - \nabla f(\mathbf{x}_{k-1}) ,~  \mathbf{v}_{k-1} \rangle \nonumber \\
		& = \langle  \mathbf{v}_{k-1} - \nabla f(\mathbf{x}_{k-1}) ,~ \mathbf{x} - \mathbf{x}_{k-1} \rangle + \frac{\eta}{2} \Big[  \|\mathbf{v}_{k-1}  \|^2  + \|  \mathbf{v}_{k-1} - \nabla f(\mathbf{x}_{k-1}) \|^2 -  \| \nabla f(\mathbf{x}_{k-1}) \|^2 \Big] \nonumber 
	\end{align*}
	where the last equation is because $2\langle\mathbf{a}, \mathbf{b} \rangle = \| \mathbf{a}\|^2 + \|\mathbf{b}\|^2 - \|\mathbf{a}- \mathbf{b} \|^2$. Since $\mathbf{v}_0 = \nabla f(\mathbf{x}_0)$, we have $\langle \mathbf{v}_0 - \nabla f(\mathbf{x}_0),  \mathbf{x} - \mathbf{x}_0 \rangle = 0$. And the proof is completed, after taking expectation and unrolling $ \langle  \mathbf{v}_{k-1} - \nabla f(\mathbf{x}_{k-1}) , \mathbf{x} - \mathbf{x}_{k-1} \rangle$. 
\end{proof}

In order to prove Theorem \ref{thm.sarah}, we need to borrow the following result from \citep{nguyen2017}.

\begin{lemma}\label{lemma.sarah_v_norm}
	\citep[Theorem 1b]{nguyen2017} If Assumptions \ref{as.1} and \ref{as.4} hold, 
	with $\eta \leq 2/(\mu+L)$, SARAH guarantees
	\begin{align*}
		\mathbb{E}\big[ \| \mathbf{v}_k \|^2 \big] \leq \bigg( 1 - \frac{2 \eta L}{1 + \kappa}\bigg)^k \mathbb{E} \big[ \| \nabla f(\mathbf{x}_0) \|^2 \big].
	\end{align*}
\end{lemma}

\noindent\textbf{Proof of Theorem \ref{thm.sarah}.}
\begin{proof}
With the choices of $\mu_0$, $\Phi_0^*$ and $\delta_k$ as in Lemma \ref{lemma.est_seq2}, we can directly apply Lemma \ref{lemma.est_seq2} to confirm that
	\begin{align*}
		& ~~ (1-\delta)^{k-1} \big[ f(\mathbf{x}_0) - f(\mathbf{x}^*) \big] + \delta \sum_{\tau=1}^{k-1} (1 - \delta)^{k-\tau-1}	\big[ f(\mathbf{x}_\tau) - f(\mathbf{x}^*) \big]  \\
		& \leq (1- \delta)^k \big[ \Phi_0(\mathbf{x}^*) - f(\mathbf{x}^*) \big]  + \frac{\mu \eta^2}{2} \sum_{\tau=1}^k (1-\delta)^{k- \tau} \| \mathbf{v}_{\tau-1} \|^2 + \sum_{\tau = 1}^k \delta (1-\delta)^{k- \tau} \langle  \mathbf{v}_{\tau-1} - \nabla f(\mathbf{x}_{\tau-1}), \mathbf{x}^* - \mathbf{x}_{\tau-1}  \rangle \nonumber \\
		& = (1- \delta)^k \big[ \Phi_0(\mathbf{x}^*) - f(\mathbf{x}^*) \big]  + \frac{\mu \eta^2}{2} \sum_{\tau=1}^k (1-\delta)^{k- \tau} \| \mathbf{v}_{\tau-1} \|^2 + \sum_{\tau = 2}^k \delta (1-\delta)^{k- \tau} \langle  \mathbf{v}_{\tau-1} - \nabla f(\mathbf{x}_{\tau-1}), \mathbf{x}^* - \mathbf{x}_{\tau-1}  \rangle \nonumber
	\end{align*}
where the last equation holds because $\mathbf{v}_0 = \nabla f(\mathbf{x}_0)$. Since $\Phi_0(\mathbf{x}^*) = f(\mathbf{x}_0)+\frac{\mu}{2}\| \mathbf{x}_0 - \mathbf{x}^* \|^2 \leq f(\mathbf{x}_0) + \frac{1}{2\mu} \| \nabla f(\mathbf{x}_0) \|^2$ and $(1-\delta)^{k-1} > (1-\delta)^k$, we can eliminate $(1 - \delta)^{k-1} \mathbb{E}[  f(\mathbf{x}_0) -  f(\mathbf{x}^*)  ] $ on the LHS, to obtain the inequality
	\begin{align*}
		& ~~ \delta \sum_{\tau=1}^{k-1} (1 - \delta)^{k-\tau-1}	\big[ f(\mathbf{x}_\tau) - f(\mathbf{x}^*) \big]  \\
		& \leq \frac{ (1- \delta)^k }{2\mu} \| \nabla f(\mathbf{x}_0) \|^2  + \frac{\mu \eta^2}{2} \sum_{\tau=1}^k (1-\delta)^{k- \tau} \| \mathbf{v}_{\tau-1} \|^2 + \sum_{\tau = 2}^k \delta (1-\delta)^{k- \tau} \langle  \mathbf{v}_{\tau-1} - \nabla f(\mathbf{x}_{\tau-1}), \mathbf{x}^* - \mathbf{x}_{\tau-1}  \rangle \nonumber.
	\end{align*}
	Taking expectation on both sides, we arrive at
	\begin{align}\label{eq.thm_sarah_eq1}
		& ~~0\leq \delta \sum_{\tau=1}^{k-1} (1 - \delta)^{k-\tau-1}	\mathbb{E}\big[ f(\mathbf{x}_\tau) - f(\mathbf{x}^*) \big]  \\
		& \leq \frac{ (1- \delta)^k }{2\mu} \mathbb{E} \big[ \| \nabla f(\mathbf{x}_0) \|^2 \big]+   \frac{\mu \eta^2}{2} \sum_{\tau=1}^k (1-\delta)^{k- \tau} \mathbb{E}\big[ \| \mathbf{v}_{\tau-1} \|^2 \big] + \sum_{\tau = 2}^k \delta (1-\delta)^{k- \tau}\mathbb{E} \big[ \langle  \mathbf{v}_{\tau-1} - \nabla f(\mathbf{x}_{\tau-1}), \mathbf{x}^* - \mathbf{x}_{\tau-1}  \rangle \big] \nonumber \nonumber \\
		& = \frac{ (1 \!-\! \delta)^k }{2\mu} \mathbb{E} \big[ \| \nabla f(\mathbf{x}_0) \|^2 \big]+  \frac{\mu \eta^2}{2} \sum_{\tau=1}^k (1\!-\!\delta)^{k- \tau} \mathbb{E}\big[ \| \mathbf{v}_{\tau\!-\!1} \|^2 \big] + \sum_{\tau = 1}^{k-1} \delta (1 \!-\! \delta)^{k\! - \!1 \!- \! \tau}\mathbb{E} \big[ \langle  \mathbf{v}_{\tau} - \nabla f(\mathbf{x}_{\tau}), \mathbf{x}^* - \mathbf{x}_{\tau}  \rangle \big] \nonumber \\
		& \leq \frac{ (1 \!-\! \delta)^k }{2\mu} \mathbb{E} \big[ \| \nabla f(\mathbf{x}_0) \|^2 \big]+   \frac{\mu \eta^2}{2} \! \sum_{\tau=1}^k (1\!-\! \delta)^{k \!-\! \tau} \mathbb{E}\big[ \| \mathbf{v}_{\tau-1} \|^2 \big] \nonumber \\
		& ~~~~~~~~~~~~~~~~~~~~~~~~~ +\frac{\delta \eta}{2} \sum_{\tau = 1}^{k-1} (1\!-\! \delta)^{k\!-\! 1 \!-\! \tau} \sum_{j=0}^{\tau-1} \mathbb{E}\Big[ \| \mathbf{v}_j - \nabla f(\mathbf{x}_j) \|^2 + \| \mathbf{v}_j  \|^2 - \| \nabla f(\mathbf{x}_j ) \|^2 \Big] \nonumber
	\end{align}
where for the last inequality we used Lemma \ref{lemma.sarah_addition}. Changing the summation order in the last term of the RHS of \eqref{eq.thm_sarah_eq1}, yields 
	\begin{align}\label{eq.thm_sarah_eq2}
		& ~~~ \frac{\delta \eta}{2} \sum_{\tau = 1}^{k-1} (1-\delta)^{k-1- \tau} \sum_{j=0}^{\tau-1} \mathbb{E}\Big[ \| \mathbf{v}_j - \nabla f(\mathbf{x}_j) \|^2 + \| \mathbf{v}_j  \|^2 - \| \nabla f(\mathbf{x}_j ) \|^2 \Big]  \nonumber \\ 
		& =  \frac{\delta \eta}{2} \sum_{\tau = 0}^{k-2} \mathbb{E}\Big[ \| \mathbf{v}_\tau - \nabla f(\mathbf{x}_\tau) \|^2 + \| \mathbf{v}_\tau  \|^2 - \| \nabla f(\mathbf{x}_\tau ) \|^2 \Big] \bigg[ \sum_{j=0}^{k-\tau -2} (1-\delta)^\tau \bigg]
		\nonumber \\
		& \leq \frac{\eta}{2} \sum_{\tau = 0}^{k-2} \bigg( \mathbb{E} \big[  \| \mathbf{v}_\tau - \nabla f(\mathbf{x}_\tau) \|^2 \big] +  \mathbb{E} \big[  \| \mathbf{v}_\tau  \|^2 \big] \bigg) - \frac{\eta}{2} \sum_{\tau = 0}^{k-2} \big( 1 - (1 - \delta)^{k-\tau-1} \big) \mathbb{E}\big[ \| \nabla f(\mathbf{x}_\tau) \|^2 \big].
	\end{align}
Now plugging \eqref{eq.thm_sarah_eq2} into \eqref{eq.thm_sarah_eq1}, and rearranging the terms, we find
	\begin{align*}
		& ~~ \frac{\eta}{2} \sum_{\tau = 0}^{k-2} \big( 1 - (1 - \delta)^{k-1-\tau} \big) \mathbb{E}\big[ \| \nabla f(\mathbf{x}_\tau) \|^2 \big] \nonumber \\
		& \leq \frac{ (1 \!-\! \delta)^k }{2\mu} \mathbb{E} \big[ \| \nabla f(\mathbf{x}_0) \|^2 \big]+ \frac{\mu \eta^2}{2} \sum_{\tau=1}^k (1-\delta)^{k- \tau} \mathbb{E}\big[ \| \mathbf{v}_{\tau-1} \|^2 \big] + \frac{\eta}{2} \sum_{\tau = 0}^{k-2} \bigg( \mathbb{E} \big[  \| \mathbf{v}_\tau - \nabla f(\mathbf{x}_\tau) \|^2 \big] +  \mathbb{E} \big[  \| \mathbf{v}_\tau  \|^2 \big] \bigg).
	\end{align*}
Dividing both sides by $\eta/2$ (and recalling that $\delta = \mu\eta$), we arrive at
	\begin{align}\label{eq.number}
		&  \sum_{\tau = 0}^{k-2} \big( 1 - (1 - \delta)^{k-\tau-1} \big) \mathbb{E}\big[ \| \nabla f(\mathbf{x}_\tau) \|^2 \big] \\
		&  \leq  \frac{ (1 \!-\! \delta)^k }{\mu \eta} \mathbb{E} \big[ \| \nabla f(\mathbf{x}_0) \|^2 \big] + \delta \sum_{\tau=1}^k (1-\delta)^{k- \tau} \mathbb{E}\big[ \| \mathbf{v}_{\tau-1} \|^2 \big] +  \sum_{\tau = 0}^{k-2} \bigg( \mathbb{E} \big[  \| \mathbf{v}_\tau - \nabla f(\mathbf{x}_\tau) \|^2 \big] +  \mathbb{E} \big[  \| \mathbf{v}_\tau  \|^2 \big] \bigg) \nonumber \\
		& \stackrel{(a)}{\leq} \frac{ (1 \!-\! \delta)^k }{\mu \eta} \mathbb{E} \big[ \| \nabla f(\mathbf{x}_0) \|^2 \big] +  \delta \sum_{\tau=1}^k (1\!-\!\delta)^{k- \tau} \mathbb{E}\big[ \| \mathbf{v}_{\tau-1} \|^2 \big] + \frac{\eta L(k\!-\!1)}{2 - \eta L} \mathbb{E}\big[ \| \nabla f(\mathbf{x}_0) \|^2 \big] +  \frac{2 \!-\! 2 \eta L}{2 \!-\! \eta L} 
		 \sum_{\tau = 0}^{k-2}  \mathbb{E} \big[  \| \mathbf{v}_\tau  \|^2 \big] \nonumber \\
		 & \stackrel{(b)}{\leq}\frac{ (1 \!-\! \delta)^k }{\mu \eta} \mathbb{E} \big[ \| \nabla f(\mathbf{x}_0) \|^2 \big]\! +\!  \delta \sum_{\tau=1}^k (1\!-\!\delta)^{k \!-\! \tau} \mathbb{E}\big[ \| \mathbf{v}_{\tau\!-\!1} \|^2 \big] \!+\! \frac{\eta L(k\!-\!1)}{2 - \eta L}\mathbb{E}\big[ \| \nabla f(\mathbf{x}_0) \|^2\big]\! + \! \frac{2\! -\! 2 \eta L}{2\! -\! \eta L} \frac{1\!+\!\kappa}{2 \eta L}  \mathbb{E}\big[ \| \nabla f( \mathbf{x}_0)  \|^2 \big] \nonumber \\
		 & \stackrel{(c)}{\leq}  \frac{ (1 -  \delta)^k }{\mu \eta} \mathbb{E} \big[ \| \nabla f(\mathbf{x}_0) \|^2 \big] + \bigg[ (1 - \delta)^k -\Big( 1 - \frac{2\eta L}{1 + \kappa} \Big)^k \bigg]\frac{ L + \mu}{ L - \mu} \mathbb{E}\big[\| \nabla f (\mathbf{x}_0) \|^2 \big]\nonumber \\
		 & ~~~~~~~~~~~~~~~~~~~~~~~~~~~~~~~~~~~~~~~~~~~~~~~~~~~~ + \frac{\eta L(k-1)}{2 - \eta L} \mathbb{E}\big[ \| \nabla f(\mathbf{x}_0) \|^2 \big] +  \frac{2 - 2 \eta L}{2 - \eta L} \frac{1+\kappa}{2 L \eta}    \mathbb{E}\big[\| \nabla f( \mathbf{x}_0)  \|^2 \big] \nonumber
	\end{align}
where in (a) we applied Lemma \ref{lemma.est_err} to deal with $\mathbb{E}[\| \nabla f(\mathbf{x}_\tau) - \mathbf{v}_\tau \|^2]$; in (b) we chose $\eta<1/L$ and used Lemma \ref{lemma.sarah_v_norm} to handle $\mathbb{E}[\|\mathbf{v}_\tau\|^2]$ in the last term; and the derivation of (c) is as follows. First, notice that $2\eta L/(1+\kappa)> \mu \eta = \delta$, which implies that $1 - \delta > 1 - [2\eta L/(1+\kappa)]$. Then, leveraging Lemma \ref{lemma.sarah_v_norm}, we have
	\begin{align*}
		& \delta \sum_{\tau=1}^k (1-\delta)^{k - \tau} \mathbb{E}\big[ \| \mathbf{v}_{\tau -1} \|^2 \big] \leq \delta \sum_{\tau=1}^k (1-\delta)^{k - \tau}  \bigg( 1 - \frac{2 \eta L}{1 + \kappa}\bigg)^{\tau-1} \mathbb{E} \big[ \| \nabla f(\mathbf{x}_0) \|^2 \big] \nonumber \\
		& = \bigg[ (1 - \delta)^k -\Big( 1 - \frac{2\eta L}{1 + \kappa} \Big)^k \bigg]\frac{ L + \mu}{ L - \mu} \mathbb{E}\big[\| \nabla f (\mathbf{x}_0) \|^2 \big].
	\end{align*}
To proceed, define
	\begin{align*}
		c:=  \sum_{\tau = 0}^{m-2} \big( 1 - (1 - \delta)^{m-\tau-1} \big) = (m-1) - \frac{ (1 - \delta) - (1 - \delta)^{m}}{ \delta} = m - \frac{1}{\delta} + \frac{(1-\delta)^m}{\delta}.
	\end{align*}
and select $m$ large enough so that $c>0$. Upon setting $p_k = (1 - (1-\delta)^{m-k-1})/c, \forall k = 0,1,\ldots, m-2$, and $p_{m-1} = p_m = 0$, we have
	\begin{align*}
		&  \mathbb{E} \big[ \| \nabla f(\tilde{\mathbf{x}}^s) \| \big] = \frac{1}{c}\sum_{\tau = 0}^{m-2} \big( 1 - (1 - \delta)^{m-\tau-1} \big) \mathbb{E}\big[ \| \nabla f(\mathbf{x}_\tau) \|^2 \big] \nonumber \\
		& \leq \underbrace{\bigg[ \frac{ (1\! -\!  \delta)^m }{c \mu \eta} + \bigg( (1\! -\! \delta)^m -\Big( 1 \!-\! \frac{2\eta L}{1\!+\! \kappa} \Big)^m \bigg)\frac{ L\! +\! \mu}{ c(L\! -\! \mu)} 
		 + \frac{\eta L(m \!-\! 1)}{c(2 \!-\! \eta L)} +  \frac{2 \!-\! 2 \eta L}{2\! -\! \eta L} \frac{1+\kappa}{2 c \eta L}  \bigg] }_{:= \lambda^{\texttt{SARAH}}}\mathbb{E}\big[\| \nabla f( \tilde{\mathbf{x}}^{s-1})  \|^2 \big]. \nonumber
	\end{align*}
Selecting $\eta<1/L$ and $m$ large enough to let $\lambda^{\texttt{SARAH}}<1$ establishes SARAH's linear convergence. For example, choosing $\eta = 1/(2L)$ and $m = 5 \kappa$, we have $\lambda^{\texttt{SARAH}}\approx 0.8$.
\end{proof}

\noindent\textbf{Proof of Corollary \ref{coro.sarah}}
\begin{proof}
If we choose $\eta = 1/(2L)$ and $m = 6\kappa = 3/(\mu\eta)$, we have $\delta = 1/(2\kappa)$ and $c \geq 4\kappa$, which implies that
\begin{align*}
	(1 - \mu \eta)^{\frac{1}{\mu\eta}}	\leq 0.4
\end{align*}
(Actually $(1 - \mu \eta)^{\frac{1}{\mu\eta}} \approx 1/e$ when $\mu \eta$ small enough). Using the value of $\eta$ and $m$, it can be verified that $\lambda^{\texttt{SVRG}} \leq 0.75$. This implies that ${\cal O}\big(\ln\frac{1}{\epsilon} \big)$ outer loops are needed for an $\epsilon$-accurate solution. And since $m = {\cal O}(\kappa)$, the overall complexity is ${\cal O}\big( (n+\kappa)\ln\frac{1}{\epsilon} \big)$.
\end{proof}

\section{Proofs for BB-SVRG and BB-SARAH}\label{appdx.bb_pf}

\noindent\textbf{Derivation of \eqref{eq.eta_bounds}:} 
It is clear that
\begin{align*}
	\eta^s = \frac{1}{\theta_\kappa} \frac{\| \tilde{\mathbf{x}}^{s-1} - \tilde{\mathbf{x}}^{s-2} \|^2}{ \big\langle \tilde{\mathbf{x}}^{s-1} - \tilde{\mathbf{x}}^{s-2}, \nabla f(\tilde{\mathbf{x}}^{s-1}) -  \nabla f(\tilde{\mathbf{x}}^{s-2}) \big\rangle} \leq \frac{1}{\theta_\kappa} \frac{\| \tilde{\mathbf{x}}^{s-1} - \tilde{\mathbf{x}}^{s-2} \|^2}{ \mu \| \tilde{\mathbf{x}}^{s-1} - \tilde{\mathbf{x}}^{s-2} \|^2} = \frac{1}{\theta_\kappa \mu}
\end{align*}
where the inequality follows since under Assumption \ref{as.3} (or \ref{as.4}) $\langle \nabla f (\mathbf{x}) -  \nabla f (\mathbf{y}), \mathbf{x} - \mathbf{y} \rangle \geq \mu \|\mathbf{x} - \mathbf{y} \|^2$ \citep[Theorem 2.1.9]{nesterov2004}. On the other hand, we have
\begin{align*}
	\eta^s \geq \frac{1}{\theta_\kappa} \frac{\| \tilde{\mathbf{x}}^{s-1} - \tilde{\mathbf{x}}^{s-2} \|^2}{ \|\tilde{\mathbf{x}}^{s-1} - \tilde{\mathbf{x}}^{s-2} \| \| \nabla f(\tilde{\mathbf{x}}^{s-1}) -  \nabla f(\tilde{\mathbf{x}}^{s-2}) \big\|} \geq \frac{1}{\theta_\kappa L}
\end{align*}
where the first inequality follows from the Cauchy-Schwarz inequality; and the second inequality is due to Assumption \ref{as.1}.

\subsection{Proof for Proposition \ref{prop.bbsvrg}}
For BB-SVRG, the step size $\eta^s$ changes across different inner loops. Since $\eta^s$ influences convergence, we will use $\lambda^s$ to denote the convergence rate of the inner loop $s$, that is, $\mathbb{E}[f(\tilde{\mathbf{x}}^s) - f(\mathbf{x}^*)] \leq \lambda^s \mathbb{E}[f(\tilde{\mathbf{x}}^{s-1}) - f(\mathbf{x}^*)]$.

\noindent\textbf{BB-SVRG with U-Avg:}
\begin{proof}
	From \citep{johnson2013}, we have the convergence rate is 
	\begin{align*}
		\lambda^s & = \frac{1}{\mu \eta^s(1-2\eta^s L) m} + \frac{2 \eta^s L}{1- 2\eta^s L} \stackrel{(a)}{\leq} \frac{\kappa \theta_\kappa}{m(1 - 2\kappa/\theta_\kappa)} + \frac{2 \kappa/\theta_\kappa}{ 1 - 2\kappa/\theta_\kappa}
	\end{align*}
	where (a) is due to \eqref{eq.eta_bounds}. Hence, by choosing $\theta_\kappa > 4 \kappa$ with  $\theta_\kappa = {\cal O} (\kappa)$ and $m = {\cal O}(\kappa^2)$ such that $\lambda^s <1 $, and using similar arguments as in the proof of Corollary \ref{coro.svrg}, one can readily verify that the complexity is ${\cal O}\big((n+\kappa^2)\ln\frac{1}{\epsilon} \big)$.
\end{proof}

\noindent\textbf{BB-SVRG with W-Avg:}
\begin{proof}
It follows from Theorem \ref{thm.svrg} and \eqref{eq.eta_bounds} that the convergence rate satisfies
	\begin{align*}
		\lambda^s & = \frac{1}{1 - (1 - \mu \eta^s)^{m-1}} \bigg[ \frac{(1 - \mu \eta^s)^m}{ 1 - 2 \eta^s L } + \frac{2 \mu L (\eta^s)^2   (1 - \mu \eta )^{m-1}}{1 - 2  L \eta^s } +  \frac{2  L \eta^s }{1 - 2 L \eta^s } \bigg]  \nonumber \\	
		& \leq \frac{1}{1 - \big(1 - \frac{1}{\kappa \theta_\kappa}\big)^{m-1}} \bigg[ \frac{\big(1 - \frac{1}{\kappa \theta_\kappa}\big)^m}{1-2\kappa/\theta_\kappa} + \frac{\frac{2\kappa}{(\theta_\kappa)^2} \big(1 - \frac{1}{\kappa \theta_\kappa}\big)^{m-1}}{1-2\kappa/\theta_\kappa} + \frac{2 \kappa/\theta_\kappa}{ 1 - 2\kappa/\theta_\kappa} \bigg]
	\end{align*}
	where the inequality is due to \eqref{eq.eta_bounds}. Hence, by choosing $\theta_\kappa > 4 \kappa$ with  $\theta_\kappa = {\cal O} (\kappa)$ and $m = {\cal O}(\kappa^2)$ so that $\lambda^s <1 $, and using similar arguments as in the proof of Corollary \ref{coro.svrg}, one can establish that the complexity is ${\cal O}\big((n+\kappa^2)\ln\frac{1}{\epsilon} \big)$.
\end{proof}

\subsection{Proof for Proposition \ref{prop.bbsarah}}
Also for BB-SARAH, the step size $\eta^s$ changes across different inner loops. Since here too $\eta^s$ affects convergence, we will use $\lambda^s$ to denote the convergence rate of the inner loop $s$; that is, $\mathbb{E}[\| f(\tilde{\mathbf{x}}^s)\|^2] \leq \lambda^s \mathbb{E}[ \| f(\tilde{\mathbf{x}}^{s-1})\|^2]$.

\noindent\textbf{BB-SARAH with U-Avg:}
\begin{proof}
We have from \citep{nguyen2017} that the convergence rate is 
	\begin{align*}
		\lambda^s & = \frac{1}{\mu \eta^s m} + \frac{\eta^s L}{2 - \eta^s L} \stackrel{(a)}{\leq} \frac{\kappa \theta_\kappa}{m} + \frac{ \kappa/\theta_\kappa}{ 2 - \kappa/\theta_\kappa}
	\end{align*}
where (a) is due to \eqref{eq.eta_bounds}. Hence, by choosing $\theta_\kappa > \kappa$ with  $\theta_\kappa = {\cal O} (\kappa)$ and $m = {\cal O}(\kappa^2)$ so that $\lambda^s <1 $, and using arguments similar to those in the proof of Corollary \ref{coro.sarah}, one can establish that the complexity is ${\cal O}\big((n+\kappa^2)\ln\frac{1}{\epsilon} \big)$.
\end{proof}

\noindent\textbf{BB-SARAH with L-Avg:}
\begin{proof}
Since the derivation in \citep{li2019l2s} relies on Assumption \ref{as.3}, we will first establish the convergence rate under Assumption \ref{as.4}. The proof proceeds along the lines of \citep{li2019l2s},
except for the use of Lemma \ref{lemma.sarah_v_norm} to bound $\mathbb{E}[\| \mathbf{v}_t^s \|]^2$. After a simple derivation, one can have the convergence rate
	\begin{align*}
		\lambda^s & = \frac{2 \eta^s L}{2 - \eta^s L} + 2(1 + \eta^s L) \bigg( 1 - \frac{2\eta^s L}{1 + \kappa} \bigg)^m.
	\end{align*}
	Then using \eqref{eq.eta_bounds} to upper bound $\lambda^s$, we have
	\begin{align*}
		\lambda^s  \leq \frac{ 2\kappa/\theta_\kappa}{ 2 - \kappa/\theta_\kappa} + 2 (1 + \kappa/\theta_\kappa)\bigg( 1 - \frac{2}{(1 + \kappa)\theta_\kappa} \bigg)^m .
	\end{align*}
Hence, by choosing $\theta_\kappa > 3\kappa/2$ with  $\theta_\kappa = {\cal O} (\kappa)$ and $m = {\cal O}(\kappa^2)$ so that $\lambda^s <1 $, and using arguments similar to those in the proof of Corollary \ref{coro.sarah}, one can verify that the complexity is ${\cal O}\big((n+\kappa^2)\ln\frac{1}{\epsilon} \big)$.
\end{proof}

\noindent\textbf{BB-SARAH with W-Avg:}
\begin{proof}
	From Theorem \ref{thm.sarah}, the convergence rate is 
	\begin{align*}
		\lambda^s & = \frac{ (1\! -\!  \mu \eta^s)^m }{c \mu \eta^s} + \bigg[ (1\! -\! \mu \eta^s)^m -\Big( 1 \!-\! \frac{2\eta^s L}{1\!+\! \kappa} \Big)^m \bigg] \frac{ L\! +\! \mu}{ c(L\! -\! \mu)} 
		 + \frac{\eta^s L(m \!-\! 1)}{c(2 \!-\! \eta^s L)} +  \frac{2 \!-\! 2 \eta^s L}{2\! -\! \eta^s L} \frac{1+\kappa}{2 c \eta^s L}  \nonumber \\
		 & \leq \frac{\kappa \theta_\kappa \big(1 - \frac{1}{\kappa \theta_\kappa}\big)^m}{c} + \big(1 - \frac{1}{\kappa \theta_\kappa}\big)^m \frac{L+\mu}{c(L-\mu)} + \frac{(m-1) \kappa/\theta_\kappa}{c(2 - \kappa/\theta_\kappa)} + \frac{2}{2-\kappa/\theta_\kappa} \frac{(1+\kappa)\theta_\kappa}{2c}
	\end{align*}
where $c = m - \frac{1}{\mu\eta^s} + \frac{(1-\mu\eta^s)^m}{\mu\eta^s} \geq  m - \frac{1}{\mu\eta^s} \geq m - \kappa \theta_\kappa$. With $\theta_\kappa = {\cal O} (\kappa)$ and $m = {\cal O}(\kappa^2)$ so that $c = {\cal O}(\kappa^2)$, we find that $\lambda^s <1 $. In addition, since $\eta^s < 1/L$ is still needed to guarantee convergence (cf. Theorem \ref{thm.sarah}), one must have $\theta_\kappa > \kappa$.	
\end{proof}

\section{More on Numerical Experiments}\label{apdx.simulation}

\subsection{More Numerical Tests of Section \ref{sec.avg}}\label{appdx.avg_tests}

This subsection presents additional numerical tests to support that averaging is not merely a `proof trick.' Specifically, experiments with SARAH under different types of averaging on datasets \textit{a9a} and \textit{diabetes} are showcased in Fig. \ref{fig.diff_sarah_avg}. Similar to the performance of SARAH on dataset \textit{w7a}, W-Avg is better when the step size is chosen large, while a smaller step size favors L-Avg.

\begin{figure*}[t]
	\centering
	\begin{tabular}{ccc}
	\hspace{-0.2cm}
		\includegraphics[width=5cm]{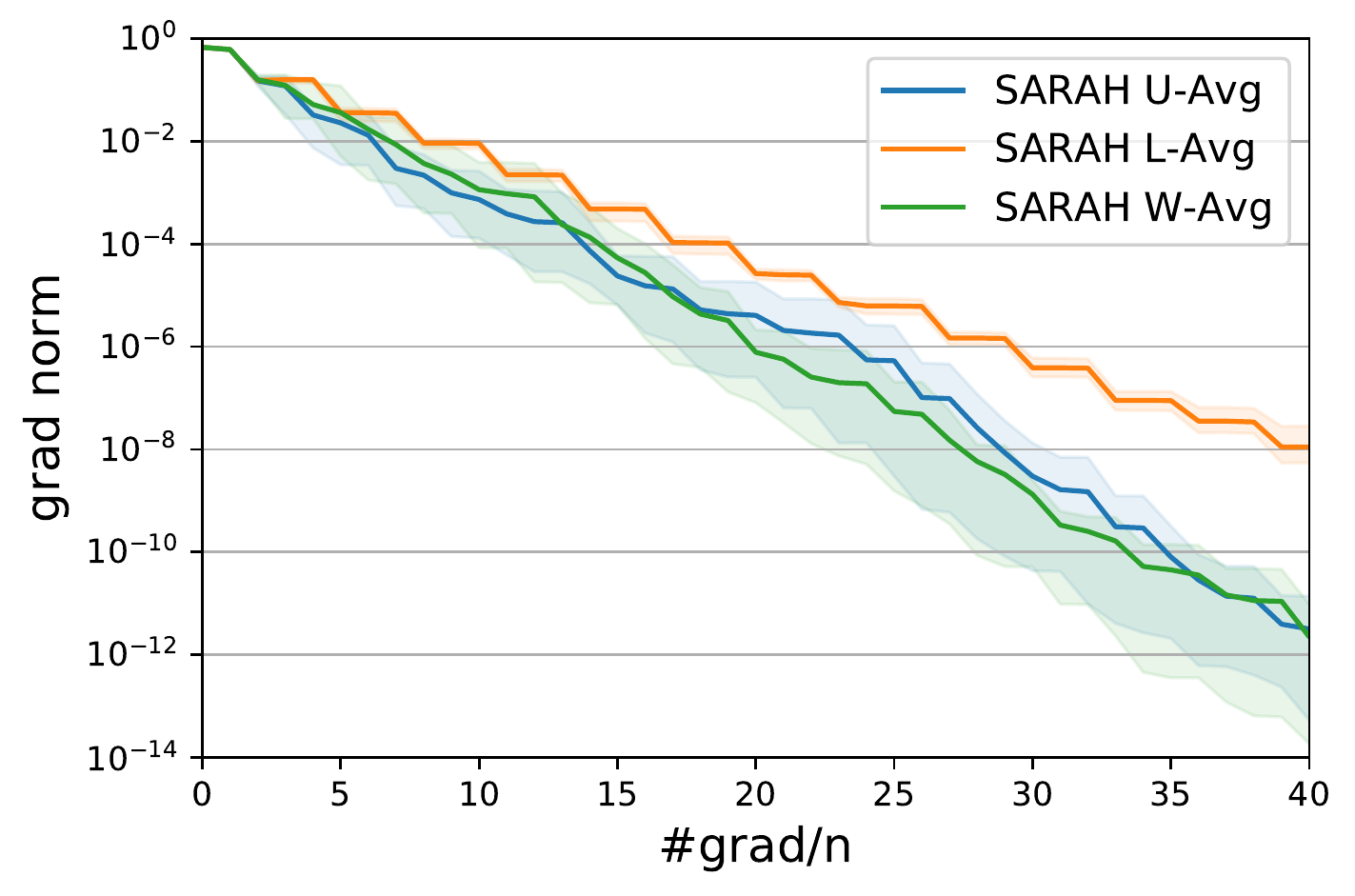}&
		\includegraphics[width=5cm]{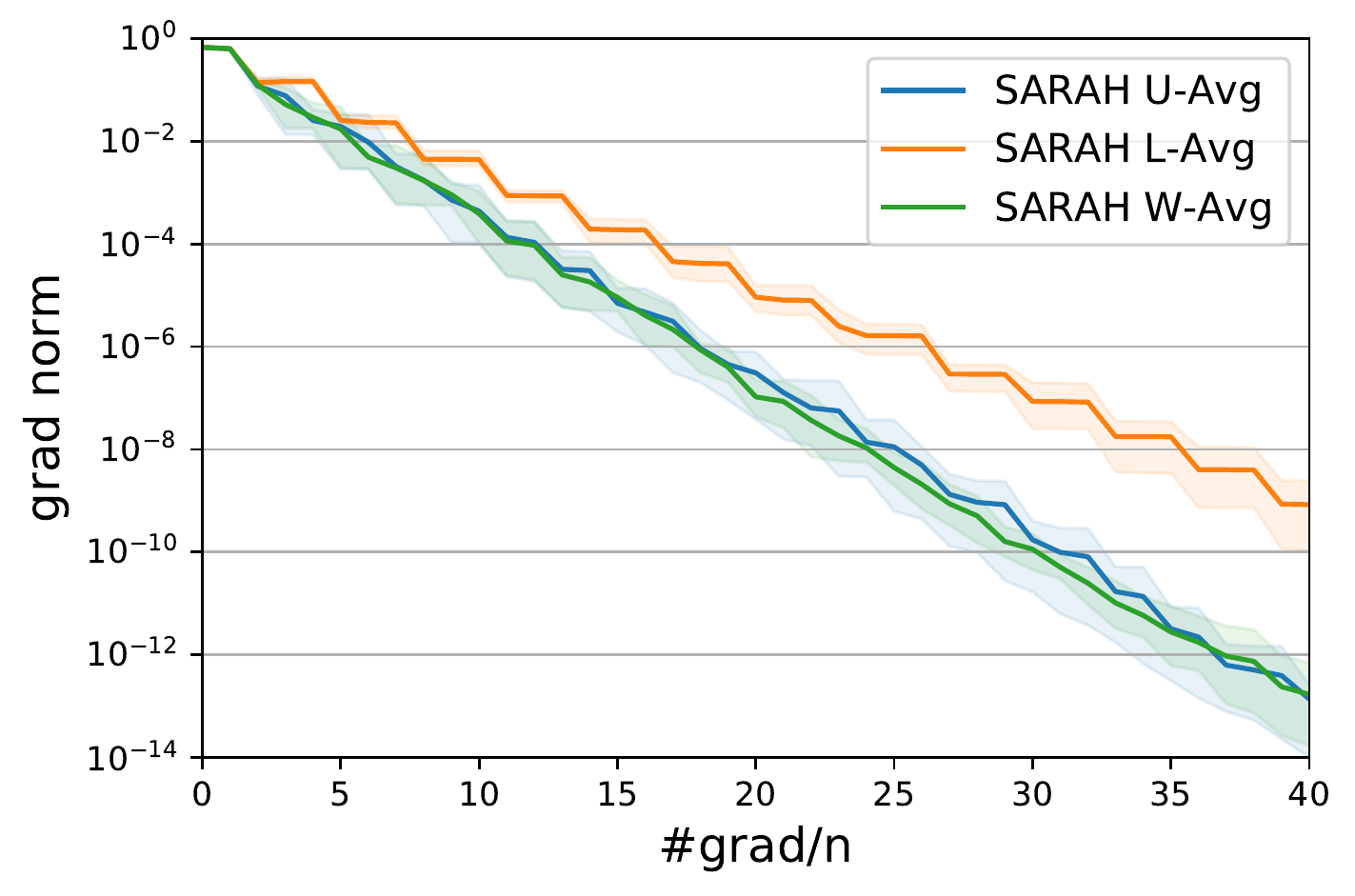}&
		\includegraphics[width=5cm]{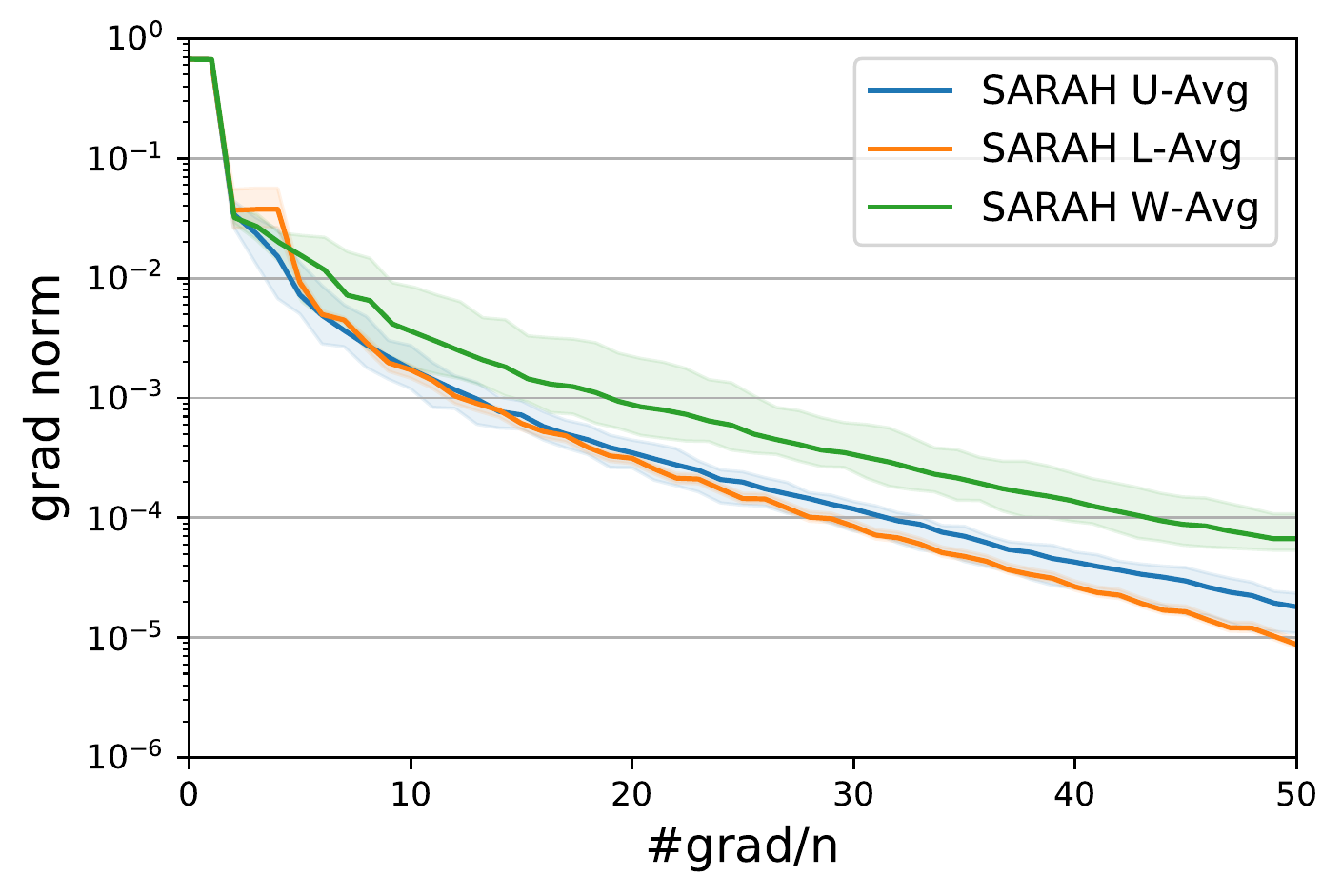}
		\\ (a) $\eta = 0.9/L$ & (b) $\eta = 0.6/L$ & (c) $\eta = 0.06/L$ \\
		\includegraphics[width=.30\textwidth]{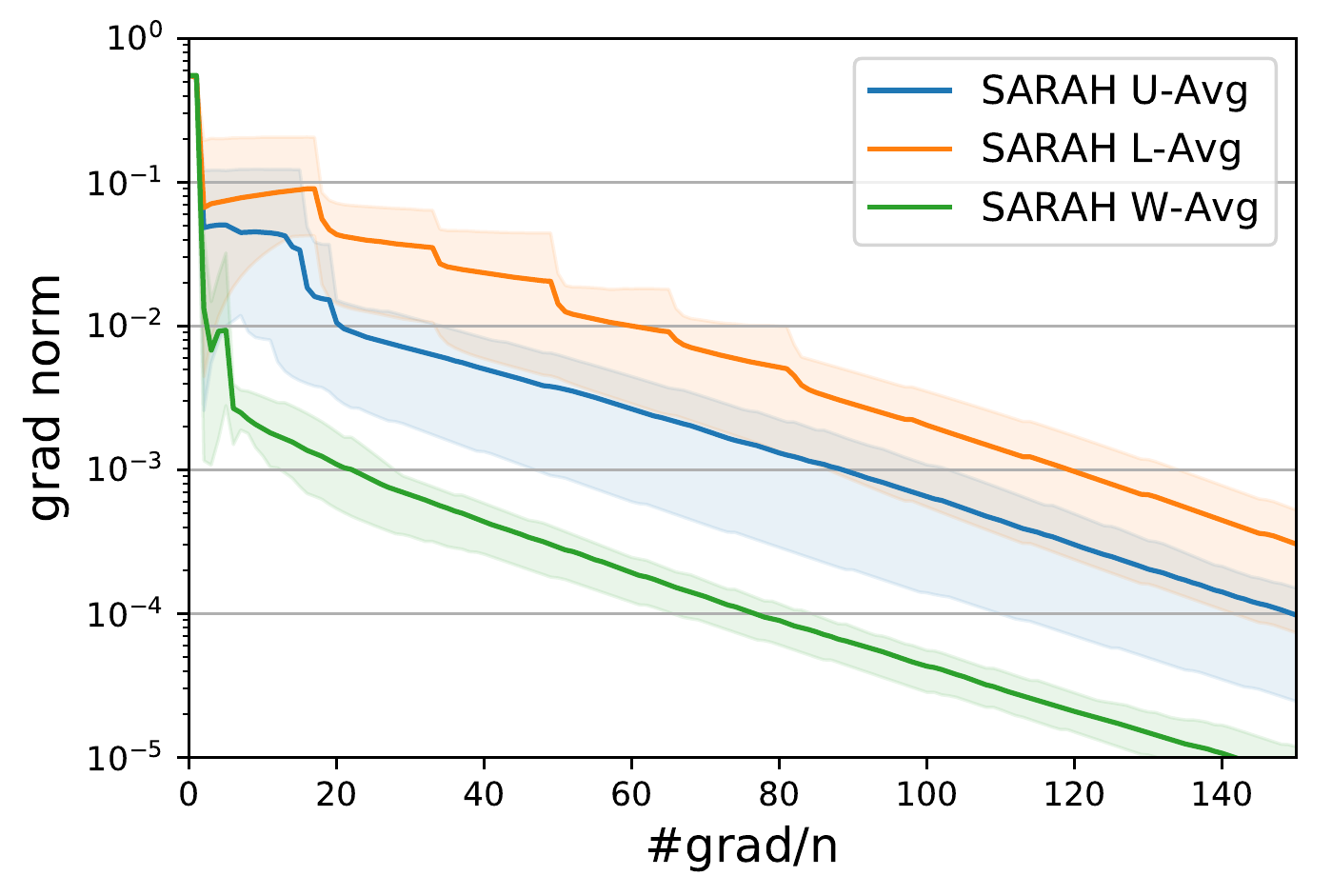}&
		\hspace{-0.2cm}
		\includegraphics[width=.30\textwidth]{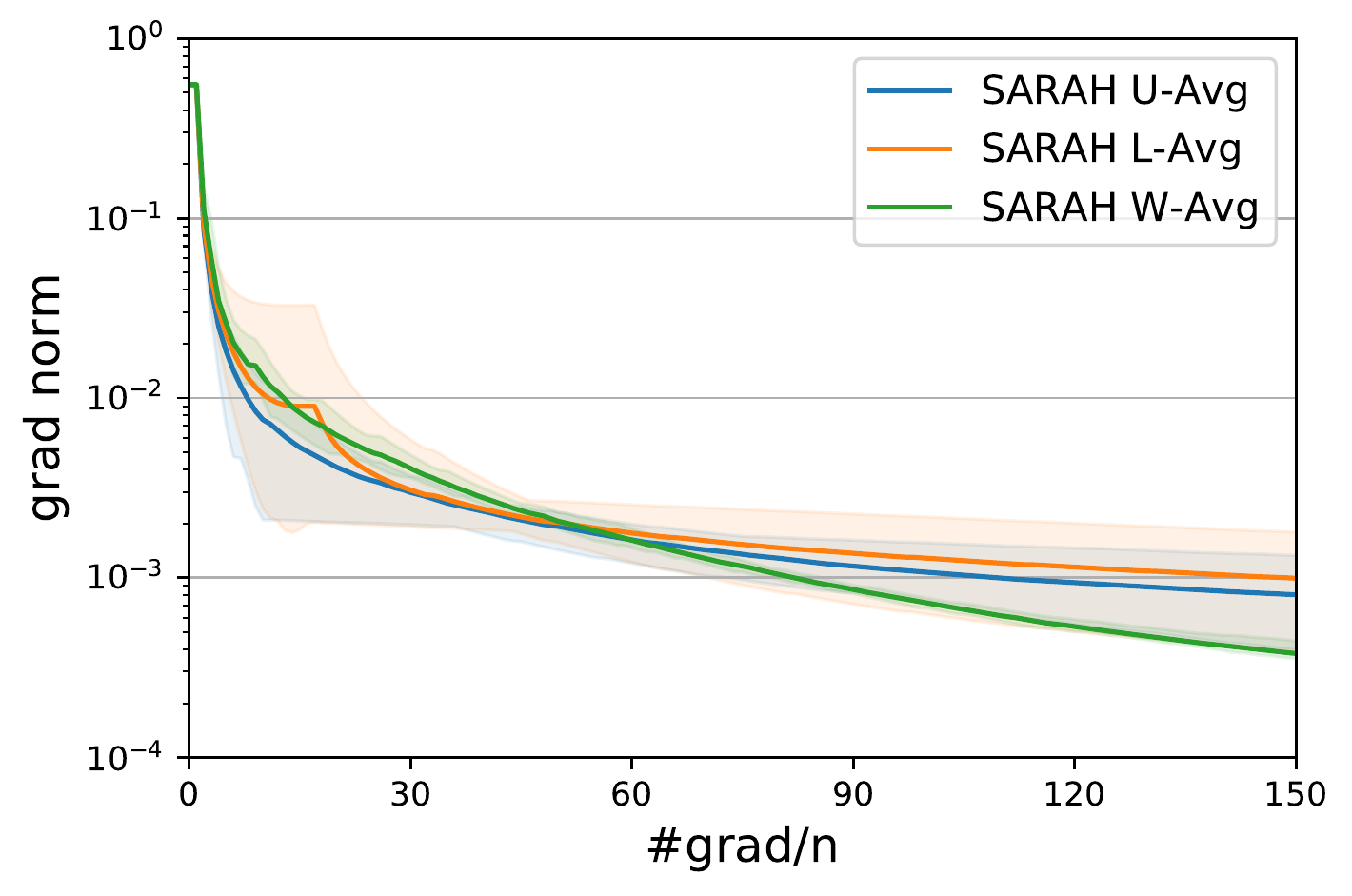}&
		\hspace{-0.3cm}
		\includegraphics[width=.30\textwidth]{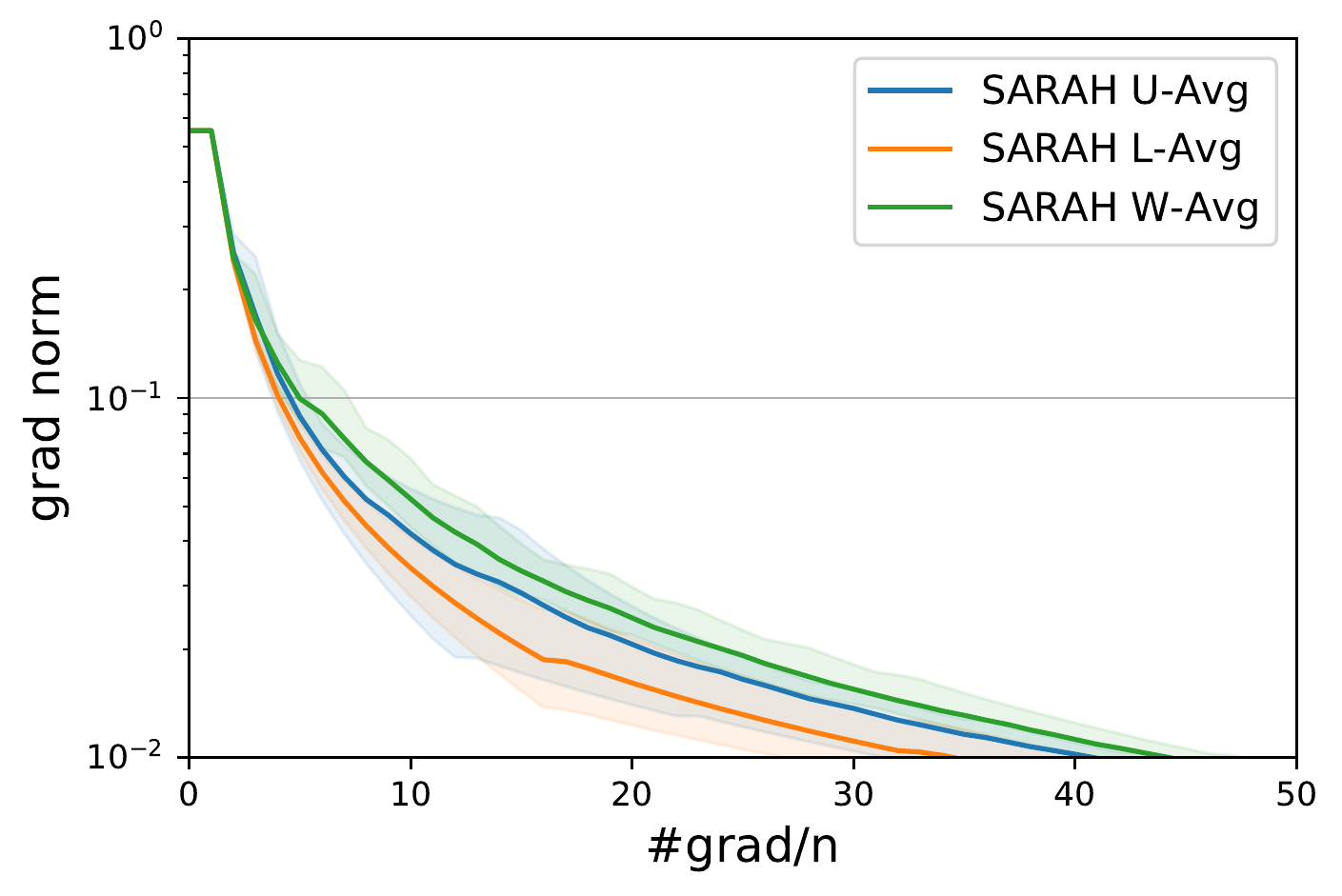}
		\\ (d) $\eta = 0.1/L$ & (e) $\eta = 0.01/L$ & (f) $\eta = 0.005/L$
	\end{tabular}
	\caption{Comparing SARAH with different types of averaging on datasets \textit{a9a} and \textit{diabetes}. In all tests, we set $\mu = 0.002$ with $m = 5\kappa$.}
	 \label{fig.diff_sarah_avg}
\end{figure*}

\subsection{Details of Datasets Used in Section \ref{sec.tests}}\label{appdx.final_tests}
The dimension $d$, number of training data $n$, the weight used for regularization, and other details of datasets used in Section \ref{sec.tests}, are listed in Table \ref{tab.para}.

\begin{table}
\centering 
\caption{Parameters of datasets used in numerical tests}\label{tab.dataset}
 \begin{tabular}{ c*{5}{|c}}
    \hline
Dataset  & $d$  & $n$ (train)  & density & $n$ (test) & $\mu$
\\ \hline
\textit{a9a}  & $122$ & $3,185$ &  $11.37\%$ & $29,376$  & $0.001$
\\ \hline
\textit{rcv1} &  $47,236$  & $20,242$ & $0.157\%$ &  $677,399$  & $0.00025$
\\ \hline
\textit{real-sim} &   $20,958$  & $50,617$ & $0.24\%$ & $21,692$ & $0.00025$
\\ \hline
\end{tabular} 
 \label{tab.para}
\end{table}

\end{document}